\documentclass[10pt,twocolumn,letterpaper]{article}

\usepackage{iccv}
\usepackage{times}
\usepackage{epsfig}
\usepackage{graphicx}
\usepackage{amsmath}
\usepackage{amssymb}

\usepackage[accsupp]{axessibility}

\usepackage{amsthm}
\newtheorem{thm}{Theorem}
\newtheorem{prop}[thm]{Proposition}
\newtheorem{cor}[thm]{Corollary}

\usepackage{algorithm}
\usepackage{listings}

\usepackage{multirow, makecell, boldline, booktabs}

\usepackage{color, colortbl} 
\definecolor{Gray}{gray}{0.875}
\definecolor{LightCyan}{rgb}{0.9,1,1}

\DeclareMathOperator{\simop}{sim}

\newcommand\blfootnote[1]{%
\begingroup
\renewcommand\thefootnote{}\footnote{#1}%
\addtocounter{footnote}{-1}%
\endgroup
}

\usepackage[pagebackref=true,breaklinks=true,letterpaper=true,colorlinks,bookmarks=false]{hyperref}

\iccvfinalcopy 

\def\iccvPaperID{4858} 

\ificcvfinal\fi

\begin{document}
\title{Nearest Neighbor Guidance for Out-of-Distribution Detection}

\author{
Jaewoo Park$^{1,2*}$ \quad Yoon Gyo Jung$^{3*}$ \quad Andrew Beng Jin Teoh$^{1\dag}$ \\
$^1$Yonsei University \quad $^2$AiV Co. \quad $^3$Northeastern University
}

\maketitle
\ificcvfinal\thispagestyle{empty}\fi

\begin{abstract}
Detecting out-of-distribution (OOD) samples are crucial for machine learning models deployed in open-world environments. Classifier-based scores are a standard approach for OOD detection due to their fine-grained detection capability. However, these scores often suffer from overconfidence issues, misclassifying OOD samples distant from the in-distribution region. To address this challenge, we propose a method called Nearest Neighbor Guidance (NNGuide) that guides the classifier-based score to respect the boundary geometry of the data manifold. NNGuide reduces the overconfidence of OOD samples while preserving the fine-grained capability of the classifier-based score. We conduct extensive experiments on ImageNet OOD detection benchmarks under diverse settings, including a scenario where the ID data undergoes natural distribution shift. Our results demonstrate that NNGuide provides a significant performance improvement on the base detection scores, achieving state-of-the-art results on both AUROC, FPR95, and AUPR metrics. The code is given at \url{https://github.com/roomo7time/nnguide}.
\end{abstract}

\blfootnote{\noindent $^{*}$Equal contribution. $^{\dag}$ Corresponding author: Andrew Beng Jin Teoh}

\section{Introduction}
\label{sec:intro}

\begin{figure}[t]
\centering
\includegraphics[width=.825\linewidth]{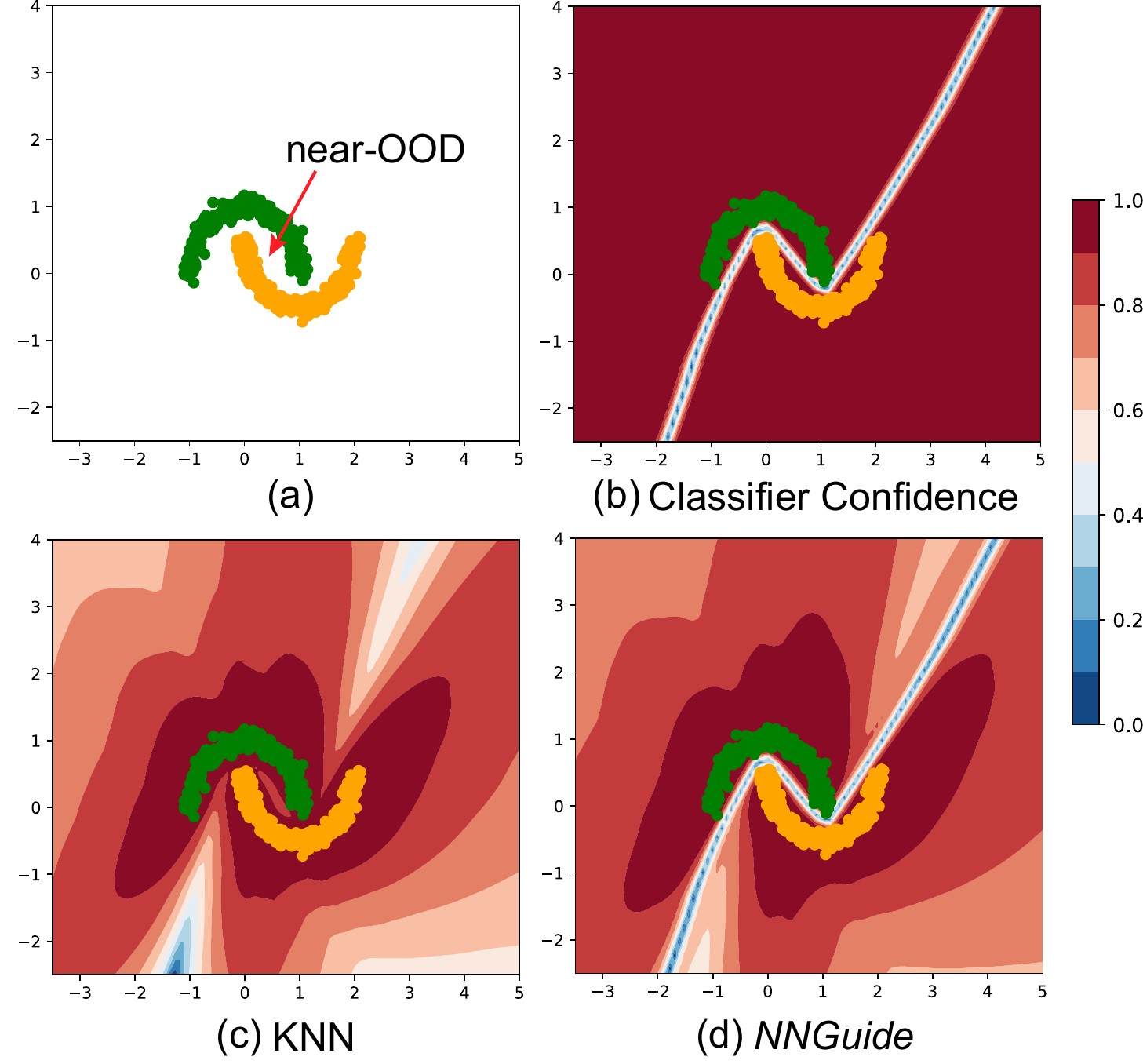}
\caption{
(a) Out-of-distribution (OOD) instances can occur in any white region, including the small area between the in-distribution (ID) classes. For instance, given 'cat' and 'dog' as ID classes, images of 'fox' will be OOD instances near the ID data. (b) The classifier-based detector (\ie confidence) assigns low scores on the small in-between area but suffers overconfidence issues. (c) Based on the distance information, KNN bounds the detection score on far-OOD regions. However, KNN lacks the fine-grained detection capability, and thus fails to detect the near-OOD. (d) NNGuide addresses both of these issues, reducing overconfidence in the far-OOD regions while achieving fine-grained detection.}
\label{fig:concept}
\end{figure}

The open-world environment poses a challenge for classification models as they may encounter input samples with unknown class labels, i.e., out-of-distribution (OOD) instances \cite{yang2022openood,pu2021anomaly,wu2021learning,bao2021evidential,chan2021entropy,hendrycks2021natural}. The detection of such anomalous examples is crucial for preventing classifier malfunctions and potential harm. As a result, in safety-critical applications like self-driving \cite{ren2021safety,chitta2021neat,vojir2021road,jung2021standardized} and biosynthesis \cite{zheng2022deep,vamathevan2019applications}, the OOD detection task plays a critical role in ensuring the dependable deployment of machine learning models. Therefore, a significant body of research has been dedicated to OOD detection \cite{yang2021generalized}.

The standard approach for OOD detection is to derive a score function from the trained network, such that the in-distribution (ID) samples exhibit relatively higher scores than OOD. One major paradigm in designing the detection score is to derive the score function based on the classifier's output signals, known as '\textit{confidence}'. Examples of classifier-based scores include maximum softmax probability \cite{hendrycks2016baseline} and energy function \cite{liu2020energy}. A major advantage of the classifier-based detection scores is their ability to fully utilize the class-dependent information of ID data and provide fine-grained detection capability. However, the classifier-based scores may suffer from overconfidence in far OOD samples, limiting their effectiveness \cite{fang2022out,hein2019relu}.

In contrast, distance-based approaches (\eg nearest neighbors \cite{sun2022out} and Mahalanobis distance \cite{lee2018simple,sehwag2021ssd}) detect OOD instances based on their distance to the ID data in the feature space. These approaches can certify low scores for far OOD regions but may not fully utilize class-dependent information, resulting in limited fine-grained detection capability. 


Our work introduces \textit{Nearest Neighbor Guidance} (NNGuide), a novel approach to improving classifier-based OOD detection scores and mitigating the issue of overconfidence. NNGuide achieves this by guiding the classifier confidence of a test input based on its similarity to its nearest neighbors in the ID bank set. As a result, NNGuide reduces the detection score in far-OOD regions while maintaining fine-grained detection capabilities.


In the toy experiment shown in Fig.~\ref{fig:concept}, the classifier-based score demonstrates its ability to assign low detection scores to samples located in the small intermediate region between the two ID classes, which is where near-OOD instances may occur. However, the score exhibits excessively high values in the outer region of the ID data. On the other hand, the distance-based score provided by KNN is effective at assigning low detection scores to far-OOD samples, but it fails to credit low scores on the small intermediate region, indicating a lack of fine-graininess. Our proposed detection score, NNGuide, mitigates the drawbacks of both methods; NNGuide assigns bounded low values for far-OOD samples while retaining fine-grained detection capability.

We conduct an extensive evaluation of NNGuide in the large-scale ImageNet-1k benchmark \cite{huang2021mos} across a variety of deep classification networks, achieving state-of-the-art results. 
Furthermore, we investigate the robustness of NNGuide by testing it on the ImageNet-1k-V2 dataset \cite{taori2020measuring,recht2019imagenet}, where the ID data undergoes natural distributional shifts. The presence of distribution shifts can lead to misidentifying ID samples as OOD and hence represents a challenging, realistic scenario.
The final part of our experiments involves an extensive ablation analysis, where we investigate the key contributing factors to the effectiveness of NNGuide, as well as its compatibility with a broad range of classifier-based scores.

\paragraph{Contributions}
The contributions of our work are summarized as follows:
\begin{itemize}
\item We propose a novel method called Nearest Neighbor Guidance (NNGuide) that guides the classifier-based detection score to reduce overconfidence in far-OOD regions while retaining its fine-grained detection capability.

\item 
We attain state-of-the-art results on the ImageNet-1k OOD detection benchmarks and demonstrate the robustness of NNGuide by considering a challenging and realistic scenario where the ID ImageNet data undergoes a natural distributional shift. 

\item
We provide an extensive and detailed ablation study, demonstrating the generality of NNGuide to a broad range of classifier-based scores.
\end{itemize}
We note that NNGuide is a post-hoc training-free inference method, and it is applicable to any standard deep classification networks.

\section{Related Works}
\label{sec:related}

The OOD detection research primarily falls into two categories: network truncation \cite{liang2017enhancing,sun2021react} and the design of a scalar score function to separate OOD instances from ID samples \cite{hendrycks2016baseline,liu2020energy}.

Network truncation \cite{liang2017enhancing,sun2021react} aims to increase the gap between ID and OOD samples by rectifying the propagated signals or weights of the network. For example, ODIN \cite{liang2017enhancing} perturbs the input signal using a gradient vector to increase the detection score, while ReAct \cite{sun2021react} clips the hidden layer activation signals using a threshold.
The clipped signals in ReAct are severely perturbed for OOD instances while being fairly retained for ID samples.
Other methods, such as DICE, RankFeat, and BATS, follow the same principle as ReAct but rectify other types of signals. 
DICE \cite{sun2022dice} sparsifies the classification layer by removing fewer contributing weights therein. RankFeat \cite{song2022rankfeat} subtracts the rank-1 approximation of the feature map from the initial feature map. BATS \cite{zhu2022boosting} cut-outs signals that deviate from the batch norm statistics.
The network truncation, however, cannot be used independently and need to be combined with a score function to detect OOD instances.


Another approach involves developing scalar score functions. 
These detection scores can be broadly categorized into two types: classifier-based and distance-based. The classifier-based scores, often referred to \textit{confidence}, leverage the classification layer of a neural network to derive the score. For example, \cite{hendrycks2016baseline} evaluated the effectiveness of the maximum output of softmax classifier probability (MSP). \cite{liu2020energy} proposed the energy function that can be viewed as a class conditional probability without bias. The maximum of logit \cite{vaze2021open,hendrycks2019scaling} captures both the class likelihood and the feature magnitude \cite{dhamija2018reducing}, and is shown to outperform the MSP counterpart. To utilize the class-dependent information extensively, \cite{hendrycks2019scaling} utilized the Kullback–Leibler (KL) divergence between the prediction and uniform distribution. GradNorm \cite{huang2021importance} on the other hand uses the norm of the gradient to minimize the KL score.

The other type of score-based approach is distance-based detectors. They identify an input sample as OOD based on its distance to the ID dataset in the feature space. One example is the Mahalanobis detector \cite{lee2018simple}, which measures the minimum distance to the class-wise means based on the shared data feature covariance.
A unified approach SSD \cite{sehwag2021ssd} operates on the same principle as Mahalanobis but instead assumes that the ID samples follow a single Gaussian distribution with a single mean. 
In contrast, KNN \cite{sun2022out} is non-parametric and therefore provides a more accurate representation of the distance to the boundary of the data manifold. CIDER \cite{ming2022exploit} shows that KNN particularly well fits to the networks with strong discriminative nature.

Classifier-based detectors exhibit low confidence scores on the class decision boundaries, and hence they are able to detect near-OOD instances around these boundaries.
However, the classifier confidence is cursed to be overly confident in the far-OOD region \cite{hein2019relu}.
Distance-based detectors on the other hand can certify low scores on the far-OOD regions.
Nevertheless, they may struggle to assign low scores to samples located in the intermediate regions around ID classes, failing to detect near-OOD instances. This limitation can be particularly problematic for parametric methods like Mahalanobis and SSD when the modeled distribution does not align well with the true data manifold.

\section{Preliminaries}
The out-of-distribution (OOD) detection is formulated as follows: Let $\mathcal{X}$ denote the input space with the output space $\mathcal{Y} = \{1, \dots, K\}$, where $K$ is the number of classes in the in-distribution (ID) dataset. Let $f: \mathcal{X} \to \mathbb{R}^{\lvert\mathcal{Y}\rvert}$ be a neural network that outputs classification logits. 
The objective of OOD detection is to devise a detection score function $S$ that determines whether a given test input $\mathbf{x} \in \mathcal{X}$ belongs to ID or OOD based on the score value $S(\mathbf{x})$:
\begin{equation}
\mathbf{x} \in 
\begin{cases}
\text{ID} &  \text{if} \quad S(\mathbf{x}) \geq \tau \\
\text{OOD} & \text{if} \quad S(\mathbf{x}) < \tau
\end{cases}
\end{equation}
The score function $S$ is either derived from the classifier outputs $f(\mathbf{x})$ or by computing the distances to the hidden layer features $\phi(\mathbf{x})$ of the network $f$.



\paragraph{Terminology}
For brevity, we call the classifier-based detection score 'confidence'.

\begin{algorithm}[t]
\caption{NNGuide Pseudocode, PyTorch-like}
\label{alg:code}
\definecolor{codeblue}{rgb}{0.25,0.5,0.5}
\definecolor{codekw}{rgb}{0.85, 0.18, 0.50}
\lstset{
  backgroundcolor=\color{white},
  basicstyle=\fontsize{7.5pt}{7.5pt}\ttfamily\selectfont,
  columns=fullflexible,
  breaklines=true,
  captionpos=b,
  commentstyle=\fontsize{7.5pt}{7.5pt}\color{codeblue},
  keywordstyle=\fontsize{7.5pt}{7.5pt}\color{codekw},
}
\begin{lstlisting}[language=python]
def nnguide(z, s, Z, S, k):
    # z: 1-by-d array of a test sample feature
    # s: a scalar base confidence of the test sample
    # Z: n-by-d array of d-dimensional features of n samples from the bank set
    # S: n-by-1 array of base confidences of n samples from the bank set
    # k: the number of nearest neighbors
    Z = normalize(Z, dim=1) 
    z = normalize(z, dim=1)

    g_topk, _ = matmul(z, (S*Z)).topk(k, dim=1)
    g = g_topk.mean(dim=1)  # the guidance term
    
    return s*g  # the guided score for the test sample
\end{lstlisting}
\end{algorithm}

\section{Method}
\label{sec:method}

\subsection{Proposed method: NNGuide}
Let $S_{base}: \mathcal{X} \to [0, \infty)$ be a given base confidence score function. We guide this base confidence by 
\begin{equation}
\label{eq:nnguide}
S_{NNGuide}(\mathbf{x}) = S_{base}(\mathbf{x}) \cdot G(\mathbf{x}).
\end{equation}
The \textit{guidance term} $G(\mathbf{x})$ is derived from the ID nearest neighbors. Particularly, let $\{\mathbf{z}_1, \dots, \mathbf{z}_n\}$ be a small bank set where $n$ is a $\alpha \%$ of features randomly sampled from the train set. $\mathbf{z}_i = \phi(\mathbf{x}_i)$ is the feature computed from a bank set sample $\mathbf{x}_i$ by the penultimate layer of the network $f$. Let $s_i = S_{base}(\mathbf{x}_i)$ denote the confidence scores of $\mathbf{x}_i$. Then for a test input $\mathbf{x}$, the guidance term $G(\mathbf{x})$ is given by the average similarity to the $k$-nearest neighbors in the bank set
\begin{equation}
\label{eq:sim_ensemble}
G(\mathbf{x}) = \frac{1}{k} \sum_{i=1}^k s_{(i)} \simop ( \mathbf{z}_{(i)}, \mathbf{z})
\end{equation}
where $\mathbf{z} = \phi(\mathbf{x})$ is the test input feature, and $\simop$ is the cosine similarity. The reordered index $(i)$ is given in the descending order of \textit{confidence-scaled} nearest neighbor similarities
\begin{equation}
\label{eq:conf_scale}
s_{(1)} \simop ( \mathbf{z}_{(1)}, \mathbf{z}) \geq \cdots \geq s_{(n)} \simop ( \mathbf{z}_{(n)}, \mathbf{z}).
\end{equation}
Due to the confidence scale term $s_i = S_{base}(\mathbf{x}_i)$, the nearest neighbors are \textit{selected in the high-confidence region}. This can enhance the utilization of more salient ID features while reducing the effect of possible outliers in ID. Overall, the confidence-scaled search makes the guided score more robust than the conventional KNN. The Pytorch-like pseudo algorithm is given in Algorithm~\ref{alg:code}.

Unless specified otherwise, we use the (negative) energy  function $S_{base}(\mathbf{x}) = - \text{Energy}(\mathbf{x})$ as the default base confidence score  due to its generality \cite{liu2020energy}.

\paragraph{Theoretical Understanding}
Let $S = S_{NNguide}$, $\widehat{\phi}(\mathbf{x}) = \phi(\mathbf{x})/\lVert \phi(\mathbf{x}) \rVert_2$, $\widehat{\mathbf{z}} = \mathbf{z} / \lVert \mathbf{z} \rVert_2$.



\begin{prop}
\label{thm:theory}
If $\min_i \lVert \widehat{\phi}(\mathbf{x}) - \widehat{\mathbf{z}}_i\rVert_2 \geq 2$, then $S(\mathbf{x}) \leq 0$. If $\lVert \widehat{\phi}(\mathbf{x}) - \widehat{\mathbf{z}}_{(k)}\rVert_2 < \epsilon$, then $S(\mathbf{x}) > M (S_{base}(\mathbf{x}) - \epsilon/2)$ if $\min_{i \leq k} s_{(i)} > M$, and $S(\mathbf{x}) \leq \delta S_{base}(\mathbf{x})$ if $\max_{i \leq k} s_{(i)} \leq \delta$.
\end{prop}

The proposition states that if a test sample $\mathbf{x}$ is far-distanced from the ID bank set in the feature space, then the guided score is certified to be low. On the other hand, if $\mathbf{x}$ is near to the ID, then the guidance up-scales the base confidence in the high-confidence region, while relatively down-scaling the base confidence around the low-confidence region (\eg class decision boundaries).
Thus, near the ID region, the guidance either retains or improves the fine-grained detection capability.

\begin{table*}
\centering
\resizebox{.995\linewidth}{!}{
\begin{tabular}{llll|lll|lll|lll}
\toprule 
Training scheme & \multicolumn{6}{c}{From scratch} & \multicolumn{6}{c}{Transfer learning} \\ 
Model & \multicolumn{3}{c}{ResNet-50} & \multicolumn{3}{c}{MobileNet} & \multicolumn{3}{c}{ViT-B/16} & \multicolumn{3}{c}{RegNet-Y/16GF} \\ 
ID accuracy$\uparrow$ & \multicolumn{3}{c}{78.73} & \multicolumn{3}{c}{72.15} & \multicolumn{3}{c}{85.3} & \multicolumn{3}{c}{86.01} \\ 
Detection method & FPR95$\downarrow$ & AUROC$\uparrow$ & AUPR$\uparrow$ & FPR95$\downarrow$ & AUROC$\uparrow$ & AUPR$\uparrow$ & FPR95$\downarrow$ & AUROC$\uparrow$ & AUPR$\uparrow$ & FPR95$\downarrow$ & AUROC$\uparrow$ & AUPR$\uparrow$ \\
\midrule
MSP (ICLR'17\cite{hendrycks2016baseline}) & 49.54 & 87.44 & 96.71 & 77.04 & 79.46 & 94.65 & 48.74 & 87.67 & 96.94 & 43.37 & 88.95 & 97.30 \\ 
MaxLogit (ICML'22\cite{hendrycks2019scaling}) & 42.12 & 90.49 & 97.57 & 78.06 & 77.99 & 94.10 & 37.62 & 89.29 & 97.11 & 26.00 & 92.91 & 98.16 \\ 
KL (ICML'22\cite{hendrycks2019scaling}) & 40.01 & 90.80 & 97.63 & 88.70 & 68.85 & 91.56 & 38.44 & 89.01 & 97.03 & 24.74 & 93.07 & 98.18 \\ 
ViM (CVPR'22\cite{wang2022vim}) & 29.86 & \textbf{93.00} & \textbf{98.13} & 76.27 & 73.59 & 92.53 & 35.16 & 91.04 & 97.73 & 21.39 & 94.89 & 98.74 \\ 
Mahalanobis (NeurIPS'22\cite{lee2018simple}) & 44.58 & 90.93 & 97.76 & \textbf{64.87} & 80.29 & 94.56 & 39.34 & 91.73 & 98.08 & 32.15 & 93.07 & 98.36 \\ 
SSD (ICLR'21\cite{sehwag2021ssd}) & 40.94 & 91.47 & 97.82 & 77.78 & 69.01 & 90.18 & 59.74 & 80.38 & 94.74 & 40.64 & 90.22 & 97.57 \\ 
GradNorm (NeurIPS'22\cite{huang2021importance}) & 28.92 & 93.00 & 98.09 & 89.88 & 65.28 & 89.88 & 38.04 & 89.28 & 96.91 & 82.86 & 62.98 & 87.76 \\ 
KNN (ICML'22\cite{sun2022out}) & 42.73 & 90.19 & 97.44 & 74.24 & 75.22 & 93.14 & 54.45 & 87.62 & 96.93 & 31.26 & 91.96 & 97.91 \\ 
Energy (NeurIPS\cite{liu2020energy}) & 40.01 & 90.80 & 97.63 & 88.70 & 68.85 & 91.56 & 38.44 & 89.01 & 97.03 & 24.73 & 93.08 & 98.18 \\ 
\rowcolor{Gray}
NNGuide (Ours) & \textbf{27.81} & 92.89 & 98.03 & 65.92 & \textbf{81.26} & \textbf{94.94} & \textbf{34.20} & \textbf{92.14} & \textbf{98.10} & \textbf{16.53} & \textbf{95.89} & \textbf{98.98} \\
\bottomrule
\end{tabular}
}
\caption{
Results on ImageNet-1k. We report the average performance across five different OODs (\ie iNaturalist, SUN, Places, Textures, OpenImage-O).
}
\label{table:result_in1k}
\end{table*}

\begin{table}[!ht]
\centering
\resizebox{.995\linewidth}{!}{
\begin{tabular}{llllll}
\toprule
Detection method & Backbone & Venue & FPR95 & AUROC & AUPR \\ 
\midrule
ODIN* & ResNet-50 & ICLR'18 & 56.48 & 85.41 & - \\ 
GODIN* & ResNet-50 & CVPR'20 & 66.07 & 82.02 & - \\ 
DICE* & ResNet-50 & ECCV'22 & 34.75 & 90.77 & - \\ 
ReAct + DICE* & ResNet-50 & ECCV'22 & 27.25 & 93.40
 & - \\ 
RankFeat* & ResNet-101 & NeurIPS'22 & 36.80 & 92.15 & - \\ 
BATS* & ResNet-50 & NeurIPS'22 & 27.11 & 94.28 & - \\ 
ASH* & ResNet-50 & Arxiv'22 & 22.73 & 95.06 & - \\ 
ReAct (+ Energy)* & ResNet-50 & NeurIPS'21 & 31.43 & 92.95 & - \\ 
ReAct + MSP & ResNet-50 & reproduced & 55.72 & 87.27 & 97.25 \\ 
ReAct + MaxLogit & ResNet-50 & reproduced & 39.97 & 91.80 & 98.29 \\ 
ReAct + KL & ResNet-50 & reproduced & 32.69 & 93.07 & 98.54 \\ 
ReAct + ViM & ResNet-50 & reproduced & 26.06 & 94.83 & 98.85 \\ 
ReAct + Mahalanobis & ResNet-50 & reproduced & 47.90 & 88.27 & 97.26 \\ 
ReAct + SSD & ResNet-50 & reproduced & 56.17 & 83.77 & 95.91 \\ 
ReAct + GradNorm & ResNet-50 & reproduced & 25.13 & 94.22 & 98.72 \\ 
ReAct + KNN & ResNet-50 & reproduced & 42.42 & 89.46 & 97.46 \\ 
ReAct + Energy & ResNet-50 & reproduced & 32.69 & 93.07 & 98.54 \\ 
\rowcolor{Gray}
ReAct + NNGuide & ResNet-50 & reproduced & \textbf{19.72} & \textbf{95.45} & \textbf{98.98} \\ 
\bottomrule
\end{tabular}
}
\caption{
The comparison with the state-of-the-art network truncation methods. The average performance across four different OODs (\ie iNaturalist, SUN, Places, and Textures) is reported. * indicates that the results are taken from the references.
}
\label{table:result_in1k_react}
\end{table}

\begin{table*}[t]
\centering
\resizebox{.9\linewidth}{!}{
\begin{tabular}{llll|lll|lll|lll}
\toprule
Training scheme& \multicolumn{6}{c}{From scratch} & \multicolumn{6}{c}{Transfer learning} \\ 
Model & \multicolumn{3}{c}{ResNet-50} & \multicolumn{3}{c}{MobileNet} & \multicolumn{3}{c}{ViT} & \multicolumn{3}{c}{RegNet} \\ 
ID accuracy$\uparrow$ & \multicolumn{3}{c}{74.43} & \multicolumn{3}{c}{67.78} & \multicolumn{3}{c}{81.16} & \multicolumn{3}{c}{82.79} \\ 
Detection score & FPR95$\downarrow$ & AUROC$\uparrow$ & AUPR$\uparrow$ & FPR95$\downarrow$ & AUROC$\uparrow$ & AUPR$\uparrow$ & FPR95$\downarrow$ & AUROC$\uparrow$ & AUPR$\uparrow$ & FPR95$\downarrow$ & AUROC$\uparrow$ & AUPR$\uparrow$ \\ 
\midrule
MSP & 55.54 & 85.30 & 84.57 & 79.40 & 76.64 & 77.14 & 54.61 & 84.93 & 84.74 & 48.99 & 86.38 & 85.93 \\ 
MaxLogit & 47.97 & 88.54 & 88.01 & 80.19 & 75.17 & 75.06 & 45.97 & 86.32 & 84.75 & 34.17 & 89.67 & 87.97 \\ 
KL & 44.67 & 88.97 & 88.24 & 89.57 & 66.17 & 67.92 & 47.03 & 85.91 & 84.41 & 32.69 & 89.74 & 87.87 \\ 
ViM & 33.11 & \textbf{91.88} & \textbf{90.75} & 76.90 & 72.32 & 72.16 & \textbf{39.13} & 89.39 & 88.74 & 28.02 & 92.80 & 92.09 \\ 
Mahalanobis & 47.85 & 89.45 & 89.65 & \textbf{65.27} & 79.37 & 78.87 & 42.99 & \textbf{90.29} & \textbf{90.76} & 36.58 & 91.72 & 91.87 \\ 
SSD & 43.21 & 90.36 & 90.09 & 77.24 & 69.32 & 67.66 & 58.28 & 81.21 & 81.39 & 44.19 & 88.73 & 88.65 \\ 
GradNorm & 32.66 & 91.87 & 90.62 & 89.98 & 64.50 & 65.27 & 47.62 & 86.07 & 83.75 & 84.47 & 61.02 & 59.56 \\ 
KNN & 44.57 & 89.22 & 88.64 & 75.49 & 74.33 & 75.20 & 57.98 & 86.52 & 86.73 & 33.74 & 91.34 & 90.54 \\ 
Energy & 44.68 & 88.97 & 88.24 & 89.57 & 66.17 & 67.92 & 47.03 & 85.91 & 84.41 & 32.86 & 89.75 & 87.89 \\ 
\rowcolor{Gray}
NNGuide & \textbf{30.78} & 91.70 & 90.08 & 67.80 & \textbf{79.40} & \textbf{78.93} & 41.73 & 90.08 & 89.95 & \textbf{21.97} & \textbf{94.17} & \textbf{93.44} \\
\bottomrule
\end{tabular}
}
\caption{
Results on ImageNet-1k-V2. The average performance across five different OODs is reported.
}
\label{table:result_in1kv2}
\end{table*}

\section{Experiments}
\label{sec:exp}
Our experiments on NNGuide are divided into the following parts: (1) We evaluate the performance of NNGuide on the standard ImageNet-1k OOD detection benchmark. (2) We examine the robustness of NNGuide against distribution shift.
In this setting, the train ID data is ImageNet-1k while the test ID data is ImageNet-1k-V2 which comprises natural distribution shift examples. (3) NNGuide is evaluated on the small-scale CIFAR-100 \cite{krizhevsky2009learning} benchmark. (4) We conduct a thorough ablation study on NNGuide to identify its key components, assess its compatibility with other classifier-based scores, and determine the optimal conditions for its use. 
\textbf{Supplementary Sec.~\ref{asec:exp} provides complete experimental results.}

\paragraph{Configuration}
NNGuide involves two hyperparameters related to the $k$-nearest neighbor search, \ie the number $k$ of nearest neighbors, and sampling ratio $\alpha \%$ to construct the bank set from the train data. In all evaluations below, we follow the guideline of \cite{sun2022out}, and use $\alpha {=}1\%$ and $k {=} 10$ to keep the balance between efficiency and performance. Extensive analysis of the hyperparameters is deferred to the ablation study in Sec.~\ref{sec:exp_ablation_hyperparameters}.

\paragraph{Comments on the computation speed}
The computational speed of nearest neighbor search has been extensively analyzed in \cite{sun2022out} for OOD detection. \cite{sun2022out} reports that KNN is as fast as or faster than most of the other detection methods in modern hardware and optimized libraries (\eg faiss). NNGuide adds no computation overhead on the nearest neighbor search algorithm.


\paragraph{Evaluation metrics}
We evaluate OOD detection methods by the widely-used metrics: the false positive rate (FPR95) when the true positive rate of ID samples is at 95\%,  the area under the receiver operating characteristic curve (AUROC), and the area under the precision-recall curve (AUPR). In all the metrics, we regard the ID samples as positive. 
In addition, we report the closed-set classification accuracy of the model on the ID dataset.

\subsection{Evaluation on the ImageNet-1k}
\label{sec:exp_in1k}

\paragraph{Datasets}
In this evaluation, the train and test ID sets are all from ImageNet-1k \cite{deng2009imagenet}. For extensiveness, the detection method is evaluated on a diverse set of OOD datasets \cite{huang2021mos}: iNaturalist \cite{van2018inaturalist}, SUN \cite{xiao2010sun}, Places \cite{zhou2017places}, Textures \cite{cimpoi2014describing}, and OpenImage-O \cite{wang2022vim}. The OOD sets have no overlapping categories with ImageNet-1k. Though there is no strict criterion to differentiate between near-OOD and far-OOD \cite{fang2022out}, \cite{yang2022openood} categorizes iNaturalist and OpenImage-O as near-OOD and Textures as far-OOD. The other two OOD sets SUN and Places have overlapping characteristics. Overall, the five OOD sets involve diverse class semantics \cite{huang2021mos}. Hence, the average performance over these OOD sets indicates the robustness of the detection method against general OOD. 

\paragraph{Backbone models}
We evaluate our proposed detection score NNGuide across four different model architectures ResNet-50 \cite{he2016deep}, MobileNet \cite{sandler2018mobilenetv2}, ViT \cite{dosovitskiy2020image}, and RegNet \cite{radosavovic2020designing}. ResNet-50 is a standard architecture for OOD evaluation, while MobileNet is a network designed particularly for efficiency. ViT  partitions an image into multiple visual tokens and processes them by a deep stack of multi-head attention blocks, whose usage has shown excellent performance in language modeling. Unlike ViT, the RegNet architecture is targeted for both efficiency and performance. RegNet is constructed by applying a network search principle at the network-population level rather than a network level, making it robust across diverse environments including distribution shifts and domain generalization \cite{cha2022domain,arpit2021ensemble}.

All four models are trained on the training fold of ImageNet-1k, and the classification layers are strictly prohibited to see any instance from OOD datasets. The first two, ResNet-50 and MobileNet, are trained from scratch on the train set. The latter two, ViT and RegNet, on the other hand, are initialized from the pretrained weights on ImageNet-21k, and then the full weights are fine-tuned on ImageNet-1k. The particular versions we use are ViT-B/16 and RegNet-Y/16GF. The transfer learning scheme for ViT and RegNet is a more practical approach due to their higher ID (closed-set) accuracy and overall better detection performance.

\vspace{-3mm}
\subsubsection{Comparison on detection scores}
We compare NNGuide with the state-of-the-art post-hoc detection score methods. The baselines include MSP, MaxLogit, ViM, Mahalanobis, SSD, GradNorm, KNN, and Energy detection scores, as described in Sec.~\ref{sec:related}. Tab.~\ref{table:result_in1k} indicates that our proposed NNGuide is more effective than or on par with other detection scores across model architectures, training schemes, and evaluation metrics. Upon the RegNet model, particularly, we achieve a new state-of-the-art performance, by significantly outperforming all other methods. This shows that with a well-trained model, NNGuide can achieve robust OOD detection on the large-scale benchmark.

\subsubsection{State-of-the-art performance with network truncator}
The recent works in OOD detection showed that the truncation methods that rectify the hidden/output layer signals give excellent performance. The insight of these approaches is that a particular truncation function perturbs only the signals from the OOD instance while retaining those of ID samples, hence enhancing the score gap between ID and OOD. The truncators however cannot be used alone and require an external OOD detection score (such as MSP, Energy, or KNN).

For a fair comparison with truncation approaches, we combine our detection score NNGuide with ReAct. Tab.~\ref{table:result_in1k_react} shows that when combined with the simple truncator i.e. ReAct, our proposed NNGuide performs the best over other recently proposed network truncators in both the FPR95 and AUROC metrics. Particularly, NNGuide outperforms BATS, which is limited to the networks with batch normalization. `ReAct + NNGuide' performs significantly better than a classifier rectifier DICE, even when used in conjunction with ReAct. In addition, When comparing ReAct combined with various detection scores, NNGuide demonstrates greater effectiveness and relevance than the other detection scores.

\subsection{Evaluation against natural distribution shift}

\paragraph{Datasets and configuration}

We evaluate the robustness of NNGuide against the natural distribution shift of the ID dataset. To this end, we consider the setting where the train ID data is ImageNet-1k and the test ID dataset is ImageNet-1k-V2 which consists of natural distribution shift samples. We note that both train and test ID datasets share the same semantics classes (\ie, the 1k number of classes in ImageNet). OOD detection in this setup can be challenging since the detection score may incorrectly identify the test ID sample as an OOD instance due to the distribution difference between the test and train ID sets. 
As to model configurations, we use the same models that are used for the ImageNet-1k evaluation in Sec.~\ref{sec:exp_in1k}.

\paragraph{Results on ImageNet-1k-V2}

Tab.~\ref{table:result_in1kv2} indicates that NNGuide is comparable to the state-of-the-art detection method ViM and Mahalanobis under MobileNet and the vision transformer, and significantly outperforms all other detection scores with ResNet-50 and RegNet. Even across MobileNet and ViT, however, NNGuide is overall more robust than Mahalanobis and ViM, indicated by the smaller fluctuation in the performance metrics. 

Both the Mahalanobis and ViM detectors are known to excel in the ViT-type architecture due to the Gaussian nature of  the vision transformer embedding space \cite{fort2021exploring,koner2021oodformer}.   We note however that ViT is suboptimal in this task even with large-scale pretraining and a much larger number of network parameters and inference time; ResNet-50 trained from scratch achieves better overall performance metrics than ViT. 

We found that the RegNet architecture built by the network-population level search principle is shown to be the best in our comparison across all metrics. Under this backbone, our proposed NNGuide outperforms other detection scores by a large margin in the FPR95 metric. 

\begin{figure}[t]
\centering
\includegraphics[width=.7\linewidth]{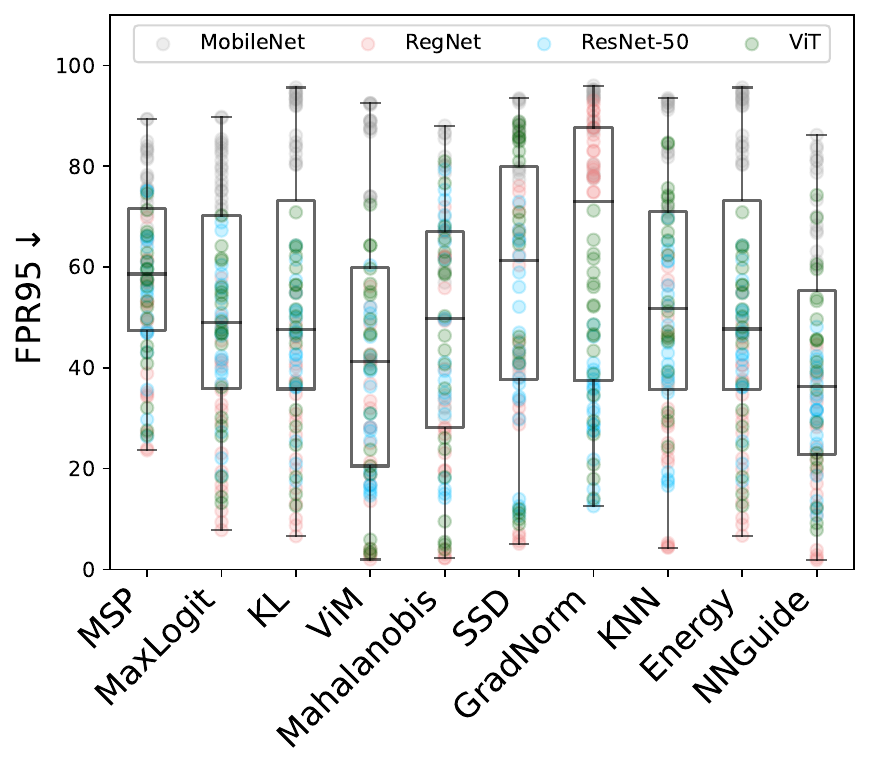}
\caption{
The performance summary of NNGuide and other detection scores on the ImageNet benchmarks across different models, test IDs, and OODs.
}
\label{fig:performance_summary}
\end{figure}

We present a summary of our performance evaluation in Fig.~\ref{fig:performance_summary}. The figure demonstrates that NNGuide exhibits less performance variation and better average performance compared to all other detection scores, including ViM, which is also based on fusing classifier signals with distance information.

\begin{table}[t]
\centering
\resizebox{.65\linewidth}{!}{
\begin{tabular}{llll}
\toprule
~ & FPR95$\downarrow$ & AUROC$\uparrow$ & AUPR$\uparrow$  \\
\midrule
MSP & 72.53 & 82.34 & 83.05  \\ 
MaxLogit & 67.92 & 85.14 & 85.61  \\ 
KL & 66.27 & 85.47 & 85.84  \\ 
ViM & 80.07 & 77.70 & 78.63  \\ 
Mahalanobis & 77.54 & 78.88 & 79.18  \\ 
SSD & 83.21 & 69.05 & 65.52  \\ 
GradNorm & 66.90 & 76.71 & 72.14  \\ 
KNN & 70.07 & 84.17 & 84.23  \\ 
Energy & 66.27 & 85.47 & 85.84  \\ 
NNGuide & \textbf{64.56} & \textbf{86.39} & \textbf{86.96} \\
\bottomrule
\end{tabular}
}
\caption{
Results on the CIFAR-100 benchmark with the ResNet-18 model trained on CIFAR-100 from scratch, achieving 75.66\% ID accuracy. The average performance across five different OODs is reported.}
\label{table:result_cifar100}
\end{table}

\subsection{Evaluation on the CIFAR-100 benchmark}

We evaluate NNGuide on the small-scale CIFAR-100 by training ResNet-18 on the train fold from scratch. Each class of CIFAR-100 contains a small number of low-resolution images, and hence the trained model can be suboptimal for both classification and OOD detection \cite{vaze2021open}. Tab.~\ref{table:result_cifar100} shows the average performance of NNGuide against five different OODs that are commonly used for evaluation (\ie CIFAR-10, SVHN, resized LSUN, resized ImageNet, and iSUN). NNGuide outperforms other baseline detection methods. However, unlike the ImageNet benchmarks, the performance boost by NNGuide is marginal on the CIFAR-100 evaluation protocol. This is due to the suboptimal model trained on low-quality ID data. Such a limitation of NNGuide is more carefully analyzed in the ablation study (Sec.~\ref{sec:exp_ablation_limitation}).

\begin{figure*}[t]
\centering
\includegraphics[width=.95\linewidth]{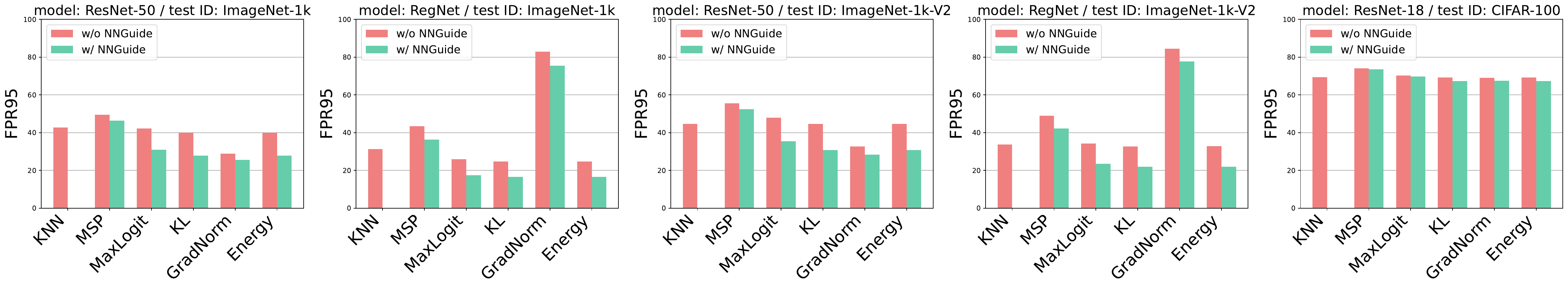}
\caption{
The compatibility to other classifier-based scores. The average performance across five different OODs is reported.
}
\label{fig:compatibility_other_confs}
\end{figure*}

\begin{table}
\centering
\resizebox{.85\linewidth}{!}{
\begin{tabular}{llll}
\toprule
~ & FPR95$\downarrow$ & AUROC$\uparrow$ & AUPR$\uparrow$ \\ 
\midrule
\multicolumn{3}{l}{\textit{baselines:}} \\
KNN & 42.73 & 90.19 & 97.44 \\ 
KNN with average similarity & 43.96 & 90.71 & 97.64 \\ 
Energy & 40.01 & 90.80 & 97.63 \\ 
\midrule
\multicolumn{3}{l}{\textit{naive fusion:}} \\
Product fusion & 33.53 & 92.17 & 97.92 \\ 
Sum fusion & 34.03 & 92.27 & 97.96 \\ 
Max fusion & 42.72 & 90.19 & 97.44 \\ 
Min fusion & 40.01 & 90.80 & 97.63 \\ 
\midrule
\multicolumn{3}{l}{\textit{missing core components}} \\
Mahalanobis guidance & 37.19 & 92.00 & 97.97 \\ 
Guidance term only & 30.31 & 91.58 & 97.42 \\ 
W/O confidence scaling & 36.23 & 91.81 & 97.87 \\ 
\midrule
NNGuide & \textbf{27.81} & \textbf{92.89} & \textbf{98.03} \\ 
\bottomrule
\end{tabular}
}
\caption{
Ablation study on the components of NNGuide. The average performance across five different OODs is reported.
}
\label{table:ablation_components}
\end{table}

\begin{figure*}[t]
\centering
\includegraphics[width=.995\linewidth]{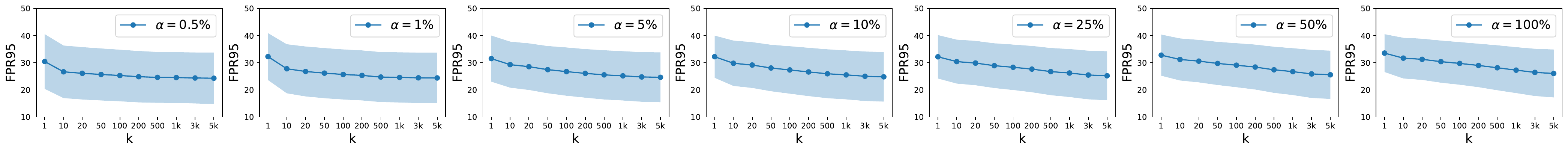}
\caption{
The hyperparameter analysis on the number $k$ of nearest neighbors and the sampling ratio $\alpha$.
}
\label{fig:hyperparameters}
\end{figure*}

\subsection{Ablation}

The ablation study is divided into several parts. (1) We examine the compatibility of NNGuide with other classifier-based detection scores besides the negative energy function. (2) We analyze which components of NNGuide contribute to improving the classifier-based confidence score. (3) We evaluate the impact of the hyperparameters $k$ and $\alpha$ related to the nearest neighbor search. (4) We present a limitation analysis of NNGuide from various aspects, highlighting the necessary requirements for its optimal usage.

\subsubsection{Compatibility to other classifier-based confidence scores}


We extend the use of NNGuide beyond the energy function by evaluating its compatibility with other classifier-based confidence scores, \ie, MSP, MaxLogit, KL, and GradNorm. Fig.~\ref{fig:compatibility_other_confs} demonstrates that NNGuide effectively enhances the performance of all considered scores.
Notably, the performance improvement is remarkable, particularly for the scores that fully utilize class-dependent information (both target and non-target class outputs) without the softmax nonlinearity. However, we note that the final performance of NNGuide heavily relies on the base confidence score. As such, NNGuide may not be effective when the base score is poor. We further analyze the limitations of NNGuide in Sec.~\ref{sec:exp_ablation_limitation}.

\subsubsection{Ablation on the components of NNGuide}
NNGuide consists of the base classifier confidence score $S_{base}(\mathbf{x})$ and the nearest-neighbor-based guidance term $G(\mathbf{x})$. The guidance term can be further broken down to the confidence scaling in Eq.~\eqref{eq:conf_scale} and the similarity ensemble in Eq.~\eqref{eq:sim_ensemble}.

Tab.~\ref{table:ablation_components} shows the overall results of evaluations conducted to ablate each component of NNGuide. Here, the original KNN detector as formulated in \cite{sun2022out} detects an OOD instance based on only the similarity to the $k$-th nearest neighbor; \ie $S_{\text{KNN}}(x) = \simop (\mathbf{z}_{(k)}, \mathbf{z})$. To see the effect of similarity ensemble, we modify the KNN detector to compute the average of top-$k$ similarities to nearest ID instances; \ie $S_{\text{avg-KNN}}(\mathbf{x}) = \sum_{i=1}^k \simop (\mathbf{z}_{(k)}, \mathbf{z})$.
As indicated by `KNN with average similarity', the similarity ensemble alone does not boost the performance. 

We argue that similarity ensemble is effective only when combined with confidence scaling. To validate this, we evaluate the performance of the guidance term $G(\mathbf{x})$. The term can be considered as a weighted KNN, where the weights are the confidences $s_i$. Indicated by 'Guidance term only', Tab.~\ref{table:ablation_components} shows a notable improvement to the original KNN. As discussed in Sec.~\ref{sec:method}, the nearest neighbor search based on \textit{confidence-scaled} similarities selects the bank set instances in the high-confidence region. Hence, this search algorithm operates with the most salient ID features, ignoring possible outliers. 

We further verify the effectiveness of the confidence-scaled nearest neighbor search by removing the confidence-scaling component from NNGuide. The result indicated by 'W/O confidence scaling' shows that confidence scaling is a significant factor in NNGuide.

We note that the Mahalanobis (density) score can also be used to bind the overconfidence of the classifier on the far-OOD region. Hence, we evaluate its impact by substituting the guiding term with the Mahalanobis distance.
Tab.~\ref{table:ablation_components} shows that the guidance by Mahalnobis score is not as effective as NNGuide. The disadvantage of Mahalanobis may stem from a strong parametric assumption that the ID features should be Gaussian. 

Finally, we test with other types of fusion techniques. We find that a naive combination of the KNN and classifier-based detection score by basic algebraic operations such as min, max, sum, and the product is not as effective as the NNGuide. Although these basic fusion approaches are aligned with our high-level objective, the confidence-scaled nearest neighbor search and similarity ensemble parts are missing therein. Hence, the naive fusion detectors could be neither fine-grained nor robust, testified by their worse performances.

\subsubsection{Analysis of the hyperparameters}
\label{sec:exp_ablation_hyperparameters}
We evaluate the impact of hyperparameters $\alpha$ and $k$ in NNGuide. We consider $\alpha \in \{0.5, 1, 5, 10, 25, 50, 100\}$ and $k \in \{1, 10, 20, 50, 100, 200, 500, 1000, 3000, 5000\}$ and  for the sampling ratio and  a number of neighbors, respectively. 
Fig.~\ref{fig:hyperparameters} indicates that the performance of NNGuide is fairly robust across different $k$ as long as $k \geq 10$. Moreover, the performance has a consistent trend across different sampling ratios $\alpha$.  This is in contrast to the vanilla KNN detection score; as reported in \cite{sun2022out}, the KNN score exhibits a degree of performance fluctuation under the variation of $k$ and $\alpha$. In addition, the performance variance is lower in the small sampling regime (\ie small $\alpha$\%), suggesting that the hyperparameters of NNGuide can be fairly easily tuned.  

\begin{figure}[t]
\centering
\includegraphics[width=.75\linewidth]{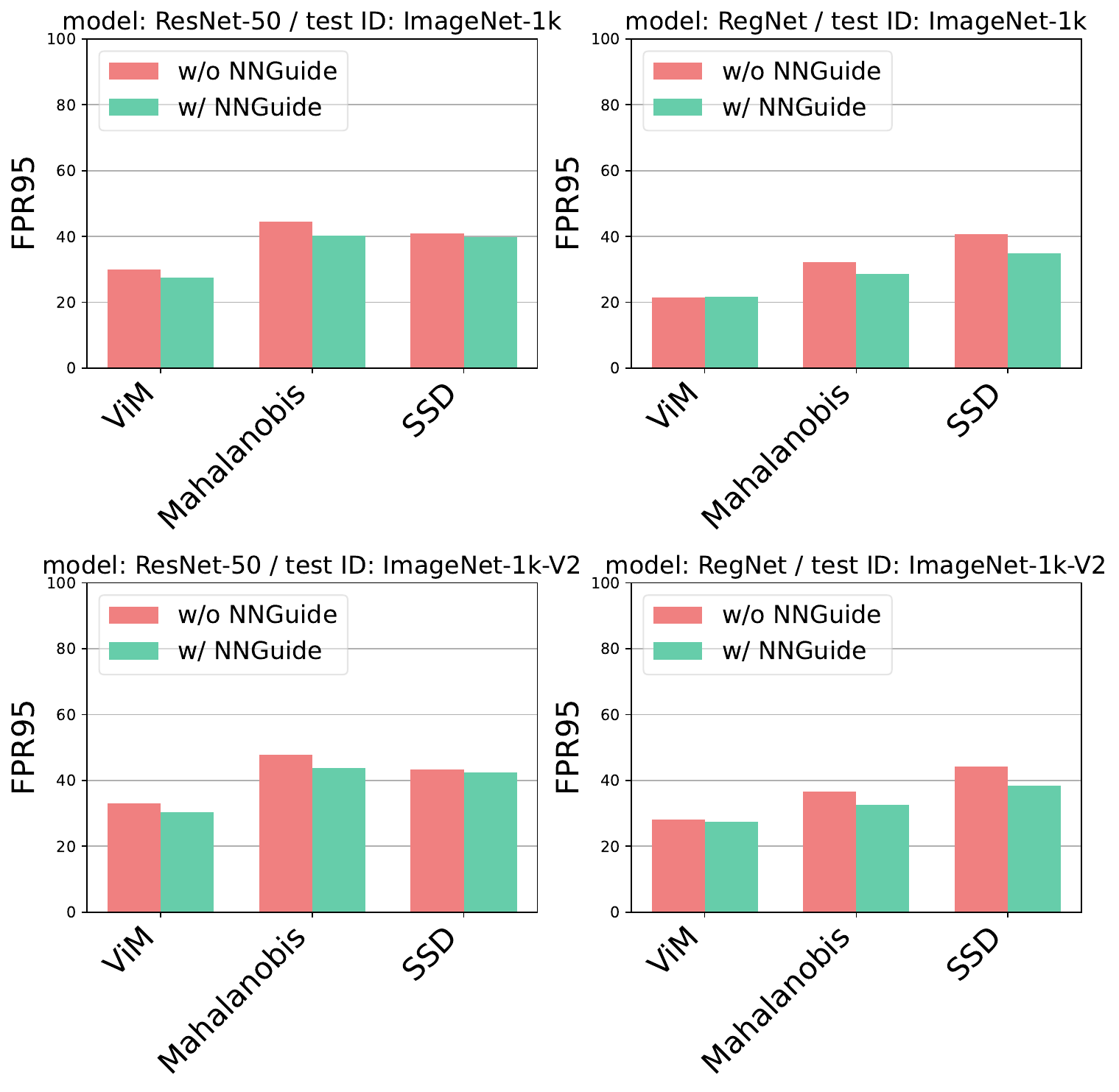}
\caption{
NNGuide could be suboptimal when combined with distance-based scores.
}
\label{fig:compatibility_dist_scores}
\end{figure}

\begin{figure}[t]
\centering
\includegraphics[width=.85\linewidth]{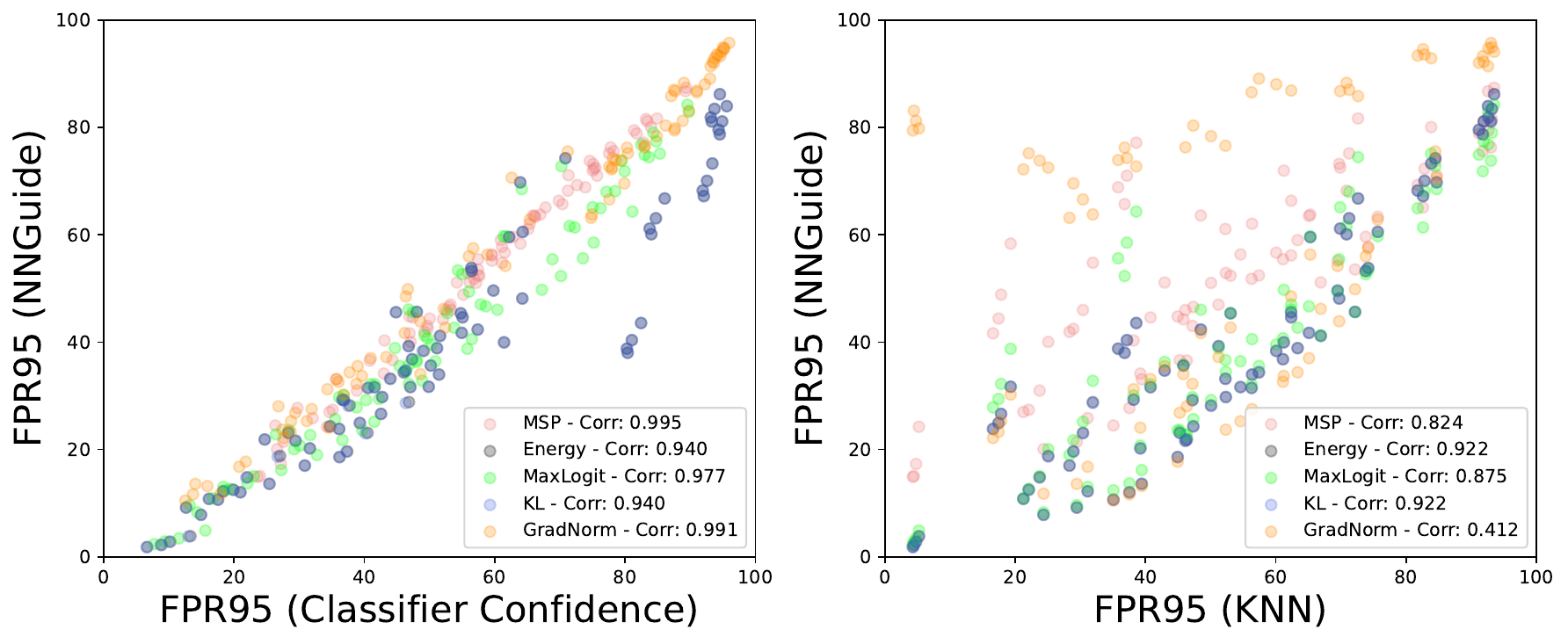}
\caption{
The dependency of NNGuide on the classifier confidence and KNN.
}
\label{fig:correlation}
\end{figure}

\subsubsection{Limitation analysis of NNGuide}
\label{sec:exp_ablation_limitation}
Despite the superiority of NNGuide compared to other detection scores, it is not perfect. We analyze its limitation, figuring out necessary requirements for NNGuide.

\paragraph{The necessity of classifier for NNGuide}
With the nearest neighbor guidance by Eq.~\eqref{eq:nnguide}, the base confidence score is assumed to be a classifier-based score with fine-grained detection capability. One possible approach is to extend the base confidence score to distance-based scores (\eg Mahalanobis, SSD, and ViM), which however may lack a fine-grained nature. We argue that NNGuide could be suboptimal if the base confidence is a distance-based score. Guiding a distance-based score by the nearest neighbor guidance may not significantly boost its detection capability. As shown in Fig.~\ref{fig:compatibility_dist_scores}, the improvement by NNGuide is inconsistent and sometimes marginal. The guidance is particularly insignificant when the base score is ViM, where the overly low energy in the far-OOD region is already mitigated by the orthogonal distance to the ID subspace. 

In some cases, NNGuide improves Mahalanobis and SSD. We believe this is due to the following fact: Mahalanobis and SSD represent poor distance functions when the data feature has deviated from Gaussian. On the other hand, the nearest neighbors represent the data boundary more accurately, providing a better distance function. Thus, the guidance by nearest neighbors refines the Mahlanobis and SSD distances.

\paragraph{Dependency on both classifier confidence and KNN}
NNGuide is formulated by nearest neighbors and the classifier outputs. Hence, the performance of NNGuide inevitably depends on both KNN and the classifier's confidence.  
Fig.~\ref{fig:correlation} indicates a strong linear correlation between NNGuide and the base classifier's confidence in terms of the performance metric.
NNGuide exhibits a strong correlation with KNN only when the classifier-based score demonstrates robust OOD detection capability. Specifically, the correlation is not observed with suboptimal classifier confidences such as GradNorm and MSP (Fig.~\ref{fig:compatibility_other_confs}). Conversely, NNGuide shows a strong correlation with KNN when the confidences are based on Energy, KL, and MaxLogit.
This correlation trend suggests that the optimal usage of NNGuide requires both good classifier confidence and a strong feature extractor. 

The analysis also explains our results on ResNet-18 (CIFAR-100) and RegNet (ImageNet). The suboptimal ResNet-18 trained on CIFAR-100 from scratch likely produces poor classification outputs and extracted features. Thus, NNGuide improves the base score only marginally.
In contrast, the RegNet model trained on large-scale datasets with the best principles attains a highly discriminative classifier and representation. Accordingly, in RegNet, the performance boost by NNGuide is significant.

\section{Conclusion}

We proposed a novel method for OOD detection called the nearest neighbor guidance (NNGuide) that improves a classifier's confidence by guiding it with the nearest neighbors' information. NNGuide prevents the overconfidence of the classifier while retaining its fine-grained detection capability, thereby achieving balanced robustness against both far- and near-OODs. NNGuide has been examined extensively on the large-scale ImageNet-1k benchmarks, including the natural distribution shift and transfer learning scenarios. NNGuide has shown to be robust across different model backbones and OODs, achieving state-of-the-art performance with RegNet.

\paragraph{Acknowledgements}
This work was supported by the National Research Foundation of Korea (NRF) grant funded by the Korea government (MSIP) (No.~NRF-2022R1A2C1010710) and the Materials/Parts Technology Development Program grant funded by the Korea government (MOTIE) (No.~1415187441).

{\small
\bibliographystyle{ieee_fullname}
\bibliography{egbib}
}

\clearpage

\appendix

\onecolumn

\setcounter{thm}{0}

\section{Supplementary to Method}

Here, we provide the proof for Proposition \ref{thm:theory}

\begin{prop}
$\quad$
\begin{enumerate}
\item[(a)] If $\min_i \lVert \widehat{\phi}(\mathbf{x}) - \widehat{\mathbf{z}}_i\rVert_2 \geq 2$, then $S(\mathbf{x}) \leq 0$.
\item[(b)] Suppose $\max_{i \leq k} \lVert \widehat{\phi}(\mathbf{x}) - \widehat{\mathbf{z}}_{(i)}\rVert_2 < \epsilon$. 
If $\min_{i \leq k} s_{(i)} > M$, then
\begin{equation}
\label{eq:near_high_conf}
S(\mathbf{x}) > M (S_{base}(\mathbf{x}) - \epsilon/2)
\end{equation}
On the other hand, if $\max_{i \leq k} s_{(i)} \leq \delta$, then
\begin{equation}
\label{eq:near_low_conf}
S(\mathbf{x}) \leq \delta  S_{base}(\mathbf{x})
\end{equation}
\end{enumerate}
\end{prop}

\begin{proof}
(a) Note that 
\begin{equation}
\label{eq:dist2sim}
 \lVert \widehat{\phi}(\mathbf{x}) - \widehat{\mathbf{z}}_i\rVert_2 =
2 - 2 \simop(\mathbf{z}_i, \mathbf{z}).
\end{equation}
Hence, if $\min_i \lVert \widehat{\phi}(\mathbf{x}) - \widehat{\mathbf{z}}_i\rVert_2 > 2$, then $\max_i \simop (\mathbf{z}_i, \mathbf{z}) \leq 0$. Therefore,
\begin{equation}
S(\mathbf{x}) = 
S_{base}(\mathbf{x}) \cdot \frac{1}{k} \sum_{i=1}^k s_{(i)} \simop (\mathbf{z}_{(i)}, \mathbf{z})
 \leq S_{base}(\mathbf{x}) \cdot 0.
\end{equation} 
(b) Assume $\max_{i \leq k} \lVert \widehat{\phi}(\mathbf{x}) - \widehat{\mathbf{z}}_{(i)}\rVert_2 < \epsilon$.
Then, $\simop (\mathbf{z}_{(i)}, \mathbf{z}) \geq 1 - \epsilon/2$ due to \eqref{eq:dist2sim}. Therefore, if $\min_{i \leq k} s_{(i)} > M$
\begin{alignat}{2}
S(\mathbf{x})  & = S_{base}(\mathbf{x}) \frac{1}{k} \sum_{i=1}^k s_{(i)} \simop (\mathbf{z}_{(i)}, \mathbf{z})  &&    \\
& \geq S_{base}(\mathbf{x})  (1 - \epsilon/2) \frac{1}{k} \sum_{i=1}^k s_{(i)}  &&   \\
& > S_{base}(\mathbf{x}) (1 - \epsilon/2) \frac{1}{k} kM  &&   \\
& =  MS_{base}(\mathbf{x})(1-\epsilon/2).  &&   \\
\end{alignat}
On the other hand, if $\max_{i \leq k} s_{(i)} \leq \delta$, then
\begin{equation}
S(\mathbf{x}) = S_{base}(\mathbf{x}) \frac{1}{k} \sum_{i=1}^k s_{(i)} \simop (\mathbf{z}_{(i)}, \mathbf{z})
\leq S_{base}(\mathbf{x}) \frac{1}{k} k \delta \cdot 1 = \delta S_{base}(\mathbf{x})
\end{equation}
as $\simop (\mathbf{z}_{(i)}, \mathbf{z}) \leq 1$. This completes the proof.
\end{proof}

\begin{cor}
\label{thm:theory_cor}
Consider $\mathbf{x}_h$ and $\mathbf{x}_l$ such that
\begin{equation}
s_{(1)_{h}} \simop ( \mathbf{z}_{(1)_{h}}, \mathbf{z}_h) \geq \cdots \geq s_{(n)_h} \simop ( \mathbf{z}_{(n)_h}, \mathbf{z}_h).
\end{equation}
and
\begin{equation}
s_{(1)_{l}} \simop ( \mathbf{z}_{(1)_{l}}, \mathbf{z}_l) \geq \cdots \geq s_{(n)_l} \simop ( \mathbf{z}_{(n)_l}, \mathbf{z}_l).
\end{equation}
where $\mathbf{z}_h = \phi(\mathbf{x}_h)$ and $\mathbf{z}_l = \phi(\mathbf{x}_l)$.
Suppose
$\max_{i \leq k} \lVert \widehat{\phi}(\mathbf{x}_h) - \widehat{\mathbf{z}}_{(i)_h}\rVert_2 < \epsilon$ and $\max_{i \leq k} \lVert \widehat{\phi}(\mathbf{x}_l) - \widehat{\mathbf{z}}_{(i)_l}\rVert_2 < \epsilon$.
Suppose 
$\min_{i \leq k} s_{(i)_h}(\mathbf{x}_h) > M$
and
$\max_{i \leq k} s_{(i)_l}(\mathbf{x}_h) \leq \delta$. Then,
\begin{equation}
\frac{S(\mathbf{x}_h)}{S_{base}(\mathbf{x}_h)}
>
\frac{S(\mathbf{x}_l)}{S_{base}(\mathbf{x}_l)}
\end{equation}
if 
\begin{equation}
M - \delta > \frac{\epsilon}{2 S_{base}(\mathbf{x}_h)}
\end{equation}
\end{cor}

\begin{proof}
Note that we have \eqref{eq:near_high_conf} and \eqref{eq:near_low_conf}.
Thus, we have 
\begin{equation}
\frac{S(\mathbf{x}_h)}{S_{base}(\mathbf{x}_h)}
> 
\frac{M(S_{base}(\mathbf{x}_h) - \epsilon/2)}{S_{base}(\mathbf{x}_h)}
\geq
\dfrac{\delta S_{base}(\mathbf{x}_l)}{S_{base}(\mathbf{x}_l)}
\geq
\frac{S(\mathbf{x}_l)}{S_{base}(\mathbf{x}_l)}
\end{equation}
if and only if
\begin{equation}
M - \delta - \frac{\epsilon}{2 S_{base}(\mathbf{x}_h)} > 0
\end{equation}
which is equivalent to $M-\delta > \frac{\epsilon}{2 S_{base}(\mathbf{x}_h)}$. This completes the proof.
\end{proof}

Corollary \ref{thm:theory_cor} states the following: Consider $\mathbf{x}_h$ in a high confidence region, and $\mathbf{x}_l$ in a relatively lower confidence region. Assume both $\mathbf{x}_h$ and $\mathbf{x}_l$ are near to the train ID data (\ie bank set). Then, the incremental factor is higher on $\mathbf{x}_h$ than $\mathbf{x}_l$ if $\mathbf{x}_h$ if the nearest neighbors to $\mathbf{x}_h$ have sufficiently high confidence.

\section{Supplementary to Experiments}
\label{asec:exp}

\begin{table*}[t]
\centering
\resizebox{.995\linewidth}{!}{
\begin{tabular}{ll|ccc|ccc|ccc|ccc|ccc|ccc|ccc}
\toprule
~ & OOD & \multicolumn{3}{c}{iNaturalist} & \multicolumn{3}{c}{SUN} & \multicolumn{3}{c}{Places} & \multicolumn{3}{c}{Textures} & \multicolumn{3}{c}{OpenImage-O} & \multicolumn{3}{c}{Average-curated} & \multicolumn{3}{c}{Average} \\
Detection score & Model & FPR95 & AUROC & AUPR & FPR95 & AUROC & AUPR & FPR95 & AUROC & AUPR & FPR95 & AUROC & AUPR & FPR95 & AUROC & AUPR & FPR95 & AUROC & AUPR & FPR95 & AUROC & AUPR \\ 
\midrule
MSP & ResNet-50 & 29.74 & 93.78 & 98.59 & 59.54 & 84.56 & 95.69 & 60.94 & 84.28 & 95.92 & 50.02 & 84.90 & 97.45 & 47.44 & 89.68 & 95.90 & 50.06 & 86.88 & 96.91 & 49.54 & 87.44 & 96.71 \\ 
MaxLogit & ResNet-50 & 22.06 & 95.99 & 99.15 & 50.90 & 88.43 & 96.99 & 53.78 & 87.37 & 96.79 & 42.25 & 88.42 & 97.90 & 41.63 & 92.23 & 97.04 & 42.25 & 90.05 & 97.71 & 42.12 & 90.49 & 97.57 \\ 
KL & ResNet-50 & 20.98 & 96.17 & 99.19 & 47.06 & 88.91 & 97.08 & 51.15 & 87.70 & 96.85 & 39.31 & 88.90 & 97.96 & 41.57 & 92.32 & 97.06 & 39.63 & 90.42 & 97.77 & 40.01 & 90.80 & 97.63 \\ 
ViM & ResNet-50 & 16.68 & 96.87 & 99.35 & 39.34 & 90.86 & 97.56 & 49.21 & 88.48 & 97.03 & 15.87 & 94.21 & 98.77 & 28.18 & 94.59 & 97.94 & 30.28 & 92.60 & 98.18 & 29.86 & \textbf{93.00} & \textbf{98.13} \\ 
Mahalanobis & ResNet-50 & 35.04 & 94.79 & 98.94 & 64.99 & 86.55 & 96.73 & 70.31 & 83.92 & 96.00 & 15.02 & 95.52 & 99.29 & 37.52 & 93.89 & 97.84 & 46.34 & 90.19 & 97.74 & 44.58 & 90.93 & 97.76 \\ 
SSD & ResNet-50 & 33.76 & 94.64 & 98.89 & 56.04 & 88.35 & 97.14 & 65.12 & 84.52 & 96.06 & 11.81 & 96.54 & 99.45 & 37.97 & 93.29 & 97.56 & 41.68 & 91.01 & 97.88 & 40.94 & 91.47 & 97.82 \\ 
GradNorm & ResNet-50 & 13.70 & 97.24 & 99.41 & 28.75 & 93.34 & 98.34 & 37.73 & 91.04 & 97.71 & 28.69 & 91.88 & 98.55 & 35.75 & 91.50 & 96.46 & 27.22 & \textbf{93.38} & \textbf{98.50} & 28.92 & 93.00 & 98.09 \\ 
KNN & ResNet-50 & 37.53 & 93.86 & 98.67 & 54.57 & 86.98 & 96.71 & 63.34 & 83.54 & 95.68 & 17.45 & 94.13 & 99.00 & 40.77 & 92.46 & 97.15 & 43.22 & 89.63 & 97.51 & 42.73 & 90.19 & 97.44 \\ 
Energy & ResNet-50 & 20.98 & 96.17 & 99.19 & 47.05 & 88.91 & 97.08 & 51.15 & 87.70 & 96.85 & 39.31 & 88.90 & 97.96 & 41.56 & 92.32 & 97.06 & 39.62 & 90.42 & 97.77 & 40.01 & 90.80 & 97.63 \\ 
\rowcolor{Gray} 
NNGuide & ResNet-50 & 12.02 & 97.47 & 99.43 & 31.62 & 91.66 & 97.63 & 38.88 & 90.12 & 97.34 & 24.93 & 91.52 & 98.27 & 31.60 & 93.66 & 97.47 & \textbf{26.86} & 92.69 & 98.17 & \textbf{27.81} & 92.89 & 98.03 \\ 
\midrule
MSP & MobileNet & 72.65 & 84.01 & 96.27 & 81.78 & 76.49 & 93.90 & 81.39 & 76.23 & 93.83 & 73.90 & 78.51 & 96.43 & 75.47 & 82.04 & 92.83 & 77.43 & 78.81 & 95.11 & 77.04 & 79.46 & 94.65 \\ 
MaxLogit & MobileNet & 76.24 & 81.80 & 95.62 & 83.00 & 74.88 & 93.31 & 82.48 & 74.72 & 93.28 & 73.55 & 77.66 & 96.21 & 75.03 & 80.90 & 92.06 & 78.82 & 77.26 & 94.60 & 78.06 & 77.99 & 94.10 \\ 
KL & MobileNet & 91.93 & 67.88 & 91.98 & 94.36 & 65.80 & 91.02 & 93.16 & 66.26 & 91.10 & 80.28 & 71.55 & 95.19 & 83.78 & 72.78 & 88.50 & 89.93 & 67.87 & 92.32 & 88.70 & 68.85 & 91.56 \\ 
ViM & MobileNet & 86.86 & 69.57 & 92.14 & 88.67 & 66.37 & 90.80 & 92.16 & 62.43 & 89.56 & 40.71 & 89.59 & 98.34 & 72.95 & 80.01 & 91.79 & 77.10 & 71.99 & 92.71 & 76.27 & 73.59 & 92.53 \\ 
Mahalanobis & MobileNet & 66.86 & 82.62 & 95.80 & 80.01 & 72.49 & 92.61 & 86.51 & 67.54 & 91.02 & 31.95 & 92.54 & 98.85 & 59.00 & 86.26 & 94.53 & \textbf{66.33} & 78.80 & 94.57 & \textbf{64.87} & 80.29 & 94.56 \\ 
SSD & MobileNet & 86.34 & 64.32 & 89.37 & 88.42 & 61.87 & 88.52 & 93.24 & 53.99 & 85.57 & 41.68 & 90.45 & 98.67 & 79.21 & 74.42 & 88.79 & 77.42 & 67.66 & 90.53 & 77.78 & 69.01 & 90.18 \\ 
GradNorm & MobileNet & 94.24 & 62.85 & 90.08 & 93.66 & 63.23 & 89.80 & 95.24 & 60.18 & 88.83 & 78.56 & 73.61 & 95.79 & 87.69 & 66.52 & 84.91 & 90.42 & 64.97 & 91.12 & 89.88 & 65.28 & 89.88 \\ 
KNN & MobileNet & 81.76 & 75.73 & 94.17 & 91.17 & 66.04 & 90.96 & 92.62 & 62.02 & 89.56 & 35.80 & 90.87 & 98.50 & 69.87 & 81.43 & 92.54 & 75.34 & 73.67 & 93.30 & 74.24 & 75.22 & 93.14 \\ 
Energy & MobileNet & 91.92 & 67.88 & 91.98 & 94.36 & 65.80 & 91.02 & 93.16 & 66.26 & 91.10 & 80.28 & 71.55 & 95.19 & 83.78 & 72.78 & 88.50 & 89.93 & 67.87 & 92.32 & 88.70 & 68.85 & 91.56 \\ 
\rowcolor{Gray} 
NNGuide & MobileNet & 68.24 & 82.07 & 95.69 & 79.57 & 76.10 & 93.86 & 81.87 & 74.23 & 93.19 & 38.78 & 89.32 & 98.18 & 61.16 & 84.58 & 93.77 & 67.11 & \textbf{80.43} & \textbf{95.23} & 65.92 & \textbf{81.26} & \textbf{94.94} \\
\midrule
MSP & ViT & 27.58 & 93.96 & 98.63 & 57.43 & 85.18 & 96.36 & 61.13 & 84.34 & 96.12 & 53.21 & 84.96 & 97.61 & 44.32 & 89.89 & 95.95 & 49.84 & 87.11 & 97.18 & 48.74 & 87.67 & 96.94 \\ 
MaxLogit & ViT & 13.19 & 97.19 & 99.37 & 47.45 & 86.80 & 96.43 & 54.36 & 83.13 & 95.20 & 44.70 & 86.11 & 97.52 & 28.41 & 93.23 & 97.02 & 39.93 & 88.31 & 97.13 & 37.62 & 89.29 & 97.11 \\ 
KL & ViT & 12.64 & 97.34 & 99.40 & 48.05 & 86.47 & 96.33 & 56.41 & 82.22 & 94.93 & 46.79 & 85.72 & 97.47 & 28.33 & 93.31 & 97.04 & 40.97 & 87.94 & 97.03 & 38.44 & 89.01 & 97.03 \\ 
ViM & ViT & 3.42 & 99.20 & 99.82 & 49.66 & 88.05 & 96.86 & 59.69 & 83.68 & 95.58 & 42.55 & 88.46 & 98.11 & 20.50 & 95.79 & 98.29 & 38.83 & 89.85 & 97.59 & 35.16 & 91.04 & 97.73 \\ 
Mahalanobis & ViT & 4.94 & 98.85 & 99.74 & 58.77 & 88.15 & 97.07 & 65.62 & 85.46 & 96.45 & 43.49 & 90.30 & 98.60 & 23.87 & 95.92 & 98.53 & 43.21 & 90.69 & 97.97 & 39.34 & 91.73 & 98.08 \\ 
SSD & ViT & 11.19 & 97.50 & 99.43 & 84.91 & 70.87 & 92.07 & 88.23 & 65.10 & 90.29 & 69.38 & 79.81 & 96.69 & 45.01 & 88.62 & 95.24 & 63.43 & 78.32 & 94.62 & 59.74 & 80.38 & 94.74 \\ 
GradNorm & ViT & 14.06 & 96.62 & 99.14 & 46.68 & 86.60 & 96.09 & 56.70 & 82.84 & 94.95 & 43.37 & 87.73 & 97.89 & 29.41 & 92.63 & 96.50 & 40.20 & 88.45 & 97.02 & 38.04 & 89.28 & 96.91 \\ 
KNN & ViT & 29.46 & 94.07 & 98.64 & 72.15 & 83.88 & 95.76 & 74.17 & 81.47 & 95.36 & 51.21 & 87.18 & 98.06 & 45.25 & 91.49 & 96.81 & 56.75 & 86.65 & 96.96 & 54.45 & 87.62 & 96.93 \\ 
Energy & ViT & 12.64 & 97.34 & 99.40 & 48.05 & 86.47 & 96.33 & 56.41 & 82.22 & 94.93 & 46.79 & 85.72 & 97.47 & 28.33 & 93.31 & 97.04 & 40.97 & 87.94 & 97.03 & 38.44 & 89.01 & 97.03 \\ 
\rowcolor{Gray} 
NNGuide & ViT & 9.17 & 97.96 & 99.55 & 45.64 & 90.03 & 97.48 & 53.82 & 87.25 & 96.75 & 39.26 & 90.01 & 98.42 & 23.11 & 95.47 & 98.28 & \textbf{36.97} & \textbf{91.31} & \textbf{98.05} & \textbf{34.20} & \textbf{92.14} & \textbf{98.10} \\ 
\midrule
MSP & RegNet & 23.62 & 94.64 & 98.75 & 52.53 & 86.58 & 96.70 & 56.83 & 85.13 & 96.32 & 49.22 & 86.47 & 97.97 & 34.65 & 91.94 & 96.75 & 45.55 & 88.21 & 97.44 & 43.37 & 88.95 & 97.30 \\ 
MaxLogit & RegNet & 7.79 & 98.03 & 99.52 & 31.68 & 91.55 & 97.79 & 41.05 & 88.08 & 96.78 & 32.73 & 91.19 & 98.64 & 16.76 & 95.68 & 98.06 & 28.31 & 92.21 & 98.18 & 26.00 & 92.91 & 98.16 \\ 
KL & RegNet & 6.58 & 98.29 & 99.58 & 29.46 & 91.85 & 97.84 & 40.71 & 87.89 & 96.70 & 30.87 & 91.51 & 98.68 & 16.10 & 95.83 & 98.10 & 26.91 & 92.39 & 98.20 & 24.74 & 93.07 & 98.18 \\ 
ViM & RegNet & 1.97 & 99.52 & 99.90 & 28.19 & 93.15 & 98.30 & 42.72 & 89.05 & 97.26 & 20.53 & 95.58 & 99.40 & 13.55 & 97.15 & 98.87 & 23.35 & 94.33 & 98.72 & 21.39 & 94.89 & 98.74 \\ 
Mahalanobis & RegNet & 2.22 & 99.36 & 99.87 & 49.30 & 89.85 & 97.54 & 61.84 & 85.77 & 96.54 & 27.91 & 93.90 & 99.15 & 19.50 & 96.48 & 98.71 & 35.32 & 92.22 & 98.28 & 32.15 & 93.07 & 98.36 \\ 
SSD & RegNet & 5.09 & 98.82 & 99.76 & 60.33 & 85.88 & 96.55 & 70.87 & 80.27 & 95.08 & 38.14 & 92.58 & 99.01 & 28.75 & 93.54 & 97.43 & 43.61 & 89.39 & 97.60 & 40.64 & 90.22 & 97.57 \\ 
GradNorm & RegNet & 87.58 & 57.01 & 87.08 & 82.97 & 65.84 & 90.21 & 91.01 & 56.04 & 86.70 & 74.81 & 75.63 & 96.20 & 77.95 & 60.39 & 78.62 & 84.09 & 63.63 & 90.05 & 82.86 & 62.98 & 87.76 \\ 
KNN & RegNet & 4.30 & 98.76 & 99.73 & 46.12 & 88.45 & 96.69 & 56.28 & 85.15 & 96.11 & 28.33 & 91.93 & 98.72 & 21.26 & 95.51 & 98.28 & 33.76 & 91.07 & 97.81 & 31.26 & 91.96 & 97.91 \\ 
Energy & RegNet & 6.68 & 98.28 & 99.57 & 29.41 & 91.88 & 97.85 & 40.51 & 87.97 & 96.72 & 30.85 & 91.48 & 98.68 & 16.19 & 95.81 & 98.09 & 26.86 & 92.40 & 98.21 & 24.73 & 93.08 & 98.18 \\ 
\rowcolor{Gray} 
NNGuide & RegNet & 1.83 & 99.57 & 99.90 & 21.58 & 94.43 & 98.58 & 31.47 & 91.87 & 97.92 & 17.00 & 95.82 & 99.42 & 10.79 & 97.73 & 99.09 & \textbf{17.97} & \textbf{95.42} & \textbf{98.96} & \textbf{16.53} & \textbf{95.89} & \textbf{98.98 }\\ 
\bottomrule
\end{tabular}
}
\caption{
Results on ImageNet-1k (ID) across five different OODs (\ie iNaturalist, SUN, Places, Textures, OpenImage-O). 'Average-curated' corresponds to the the results averaged over iNaturalist, SUN, Places, and Textures. 
}
\label{table:result_in1k_supp}
\end{table*}

\subsection{Evaluation on ImageNet-1k}

\paragraph{Backbone models}
We evaluate detection scores on four different model architectures ResNet-50 \cite{he2016deep}, MobileNet \cite{sandler2018mobilenetv2}, ViT \cite{dosovitskiy2020image}, and RegNet \cite{radosavovic2020designing}. 
\begin{itemize}
\item 
We use ResNet-50 trained on ImageNet-1k from scratch. The model can be downloaded from \url{https://github.com/deeplearning-wisc/knn-ood} in `Pre-trained model'.
\item
We use MobileNet-v2 trained on ImageNet-1k from scratch. The model can be downloaded from \url{https://pytorch.org/vision/stable/models/generated/torchvision.models.mobilenet_v2.html#torchvision.models.MobileNet_V2_Weights}.
\item
We use ViT-B/16 pretrained on ImageNet-21k and then fine-tuned on ImageNet-1k with the full weight update. The model can be downloaded from \url{https://pytorch.org/vision/stable/models/generated/torchvision.models.vit_b_16.html#torchvision.models.ViT_B_16_Weights}.
\item
We use RegNet-Y-16GF pretrained on ImageNet-21k and then fine-tuned on ImageNet-1k with the full weight update. The model can be downloaded from \url{https://pytorch.org/vision/stable/models/generated/torchvision.models.regnet_y_16gf.html#torchvision.models.RegNet_Y_16GF_Weights}.
\end{itemize}

\subsubsection{Comparison of detection scores}

\paragraph{Results}
The result on the ImageNet-1k OOD benchmark is given in Tab.~\ref{table:result_in1k_supp}.

\begin{table*}[t]
\centering
\resizebox{.995\linewidth}{!}{
\begin{tabular}{l|lll|lll|lll|lll|lll|lll|lll}
\toprule
~  & \multicolumn{3}{c}{iNaturalist} & \multicolumn{3}{c}{SUN} & \multicolumn{3}{c}{Places} & \multicolumn{3}{c}{Textures} & \multicolumn{3}{c}{OpenImage-O} & \multicolumn{3}{c}{Average on curated OODs} & \multicolumn{3}{c}{Average on all} \\
Detection method & FPR95 & AUROC & AUPR & FPR95 & AUROC & AUPR & FPR95 & AUROC & AUPR & FPR95 & AUROC & AUPR & FPR95 & AUROC & AUPR & FPR95 & AUROC & AUPR & FPR95 & AUROC & AUPR \\ 
\midrule
ODIN* & 47.66 & 89.66 & - & 60.15 & 84.59 & - & 67.89 & 81.78 & - & 50.23 & 85.62 & - & - & - & - & 56.48 & 85.41 & - & - & - & - \\ 
GODIN* & 61.91 & 85.40 & - & 60.83 & 85.60 & - & 63.70 & 83.81 & - & 77.85 & 73.27 & - & - & - & - & 66.07 & 82.02 & - & - & - & - \\ 
DICE* & 25.63 & 94.49 & - & 35.15 & 90.83 & - & 46.49 & 87.48 & - & 31.72 & 90.30 & - & - & - & - & 34.75 & 90.78 & - & - & - & - \\ 
RankFeat* & 41.31 & 91.91 & - & 29.27 & 94.07 & - & 39.34 & 90.93 & - & 37.29 & 91.70 & - & - & - & - & 36.80 & 92.15 & - & - & - & - \\ 
BATS* & 12.57 & 97.67 & - & 22.62 & 95.33 & - & 34.34 & 91.83 & - & 38.90 & 92.27 & - & - & - & - & 27.11 & 94.28 & - & - & - & - \\ 
ASH* & 14.21 & 97.32 & - & 22.08 & 95.10 & - & 33.45 & 92.31 & - & 21.17 & 95.50 & - & - & - & - & 22.73 & 95.06 & - & - & - & - \\ 
ReAct* (+ Energy) & 20.38 & 96.22 & - & 24.20 & 94.20 & - & 33.85 & 91.58 & - & 47.30 & 89.80 & - & - & - & - & 31.43 & 92.95 & - & - & - & - \\
ReAct + MSP & 44.36 & 91.62 & 98.17 & 58.46 & 86.43 & 96.65 & 63.83 & 84.51 & 96.18 & 56.24 & 86.51 & 98.01 & 57.04 & 88.40 & 95.64 & 55.72 & 87.27 & 97.25 & 55.99 & 87.49 & 96.93 \\ 
ReAct + MaxLogit & 26.47 & 95.29 & 99.00 & 39.83 & 91.74 & 98.04 & 48.18 & 89.44 & 97.42 & 45.41 & 90.75 & 98.72 & 44.82 & 91.54 & 96.88 & 39.97 & 91.80 & 98.29 & 40.94 & 91.75 & 98.01 \\ 
ReAct + KL & 19.99 & 96.31 & 99.21 & 29.67 & 93.40 & 98.38 & 39.70 & 90.95 & 97.71 & 41.42 & 91.62 & 98.84 & 41.54 & 91.85 & 96.93 & 32.69 & 93.07 & 98.54 & 34.46 & 92.82 & 98.21\\ 
ReAct + ViM & 18.87 & 96.69 & 99.32 & 32.39 & 93.82 & 98.60 & 45.45 & 90.38 & 97.67 & 7.55 & 98.45 & 99.81 & 38.72 & 92.34 & 97.15 & 26.06 & 94.83 & 98.85 & 28.60 & 94.33 & 98.51 \\ 
ReAct + Mahalanobis & 44.96 & 91.44 & 98.14 & 62.50 & 84.64 & 96.37 & 73.04 & 79.23 & 94.80 & 11.10 & 97.78 & 99.72 & 55.98 & 85.90 & 94.33 & 47.90 & 88.27 & 97.26 & 49.52 & 87.80 & 96.67\\ 
ReAct + SSD & 56.83 & 86.00 & 96.65 & 71.06 & 79.56 & 94.79 & 80.38 & 73.08 & 92.67 & 16.40 & 96.45 & 99.53 & 65.70 & 78.79 & 90.01 & 56.17 & 83.77 & 95.91 & 58.07 & 82.78 & 94.73 \\ 
ReAct + GradNorm & 14.88 & 97.01 & 99.33 & 25.54 & 94.19 & 98.57 & 36.49 & 91.12 & 97.74 & 23.60 & 94.57 & 99.23 & 39.46 & 89.67 & 95.46 & 25.13 & 94.22 & 98.72 & 27.99 & 93.31 & 98.07 \\ 
ReAct + KNN & 37.05 & 93.03 & 98.49 & 55.73 & 86.11 & 96.58 & 67.99 & 80.70 & 95.04 & 8.92 & 98.01 & 99.74 & 53.02 & 88.26 & 95.43 & 42.42 & 89.46 & 97.46 & 44.54 & 89.22 & 97.06 \\ 
ReAct + Energy & 19.99 & 96.31 & 99.21 & 29.67 & 93.40 & 98.38 & 39.70 & 90.95 & 97.71 & 41.42 & 91.62 & 98.84 & 41.54 & 91.85 & 96.93 & 32.69 & 93.07 & 98.54 & 34.46 & 92.82 & 98.21 \\ 
\rowcolor{Gray}
ReAct + NNGuide & 11.12 & 97.70 & 99.50 & 20.51 & 95.26 & 98.83 & 29.99 & 92.70 & 98.13 & 17.27 & 96.11 & 99.46 & 35.10 & 92.49 & 97.09 & \textbf{19.72} & \textbf{95.45} & \textbf{98.98} & \textbf{22.80} & \textbf{94.85} & \textbf{98.60} \\ 
\bottomrule
\end{tabular}
}
\caption{
Results on ImageNet-1k with the network truncators using ResNet-50. The 'curated OODs' are the datasets of iNaturalist, SUN, Places, and Textures.
}
\label{table:result_in1k_react_supp}
\end{table*}

\subsubsection{Compatibility with network truncator}

\subsection{Evaluation against natural distribution shift (ImageNet-1k-V2)}

\paragraph{Dataset}
ImageNet-1k-V2 \cite{taori2020measuring,recht2019imagenet} contains samples of the same semantic classes as those of original ImageNet-1k. Due to the different data collecting schemes applied on ImageNet-1k-V2, the data experiences natural distribution shift. The dataset involves three different data folds. Their differing characteristics are determined by how they are collected. 
We evaluate the performance by combining the different folds.

\paragraph{Backbone models}
We use the same backbone models used in the ImageNet-1k benchmark.

\paragraph{Setup}
The model is trained on the train fold of original ImageNet-1k. Then, during testing, the test set ID is the combined ImageNet-1k-V2 folds, and the test OOD is either of five different OODs (\ie, iNaturalist, SUN, Places, Textures, and OpenImage-O).

\paragraph{Results}
The results on the ImageNet-1k-V2 are given in Tab.~\ref{table:result_in1kv2_supp}.

\begin{table*}[t]
\centering
\resizebox{.995\linewidth}{!}{
\begin{tabular}{ll|lll|lll|lll|lll|lll|lll}
\toprule
~ & OOD & \multicolumn{3}{c}{iNaturalist} & \multicolumn{3}{c}{SUN} & \multicolumn{3}{c}{Places} & \multicolumn{3}{c}{Textures} & \multicolumn{3}{c}{OpenImage-O} & \multicolumn{3}{c}{Average} \\
Method & Model & FPR95 & AUROC & AUPR & FPR95 & AUROC & AUPR & FPR95 & AUROC & AUPR & FPR95 & AUROC & AUPR & FPR95 & AUROC & AUPR & FPR95 & AUROC & AUPR \\
\midrule
MSP & ResNet-50 & 36.31 & 92.39 & 92.64 & 65.44 & 82.07 & 80.12 & 66.23 & 81.69 & 81.13 & 55.35 & 82.73 & 87.05 & 54.39 & 87.64 & 81.91 & 55.54 & 85.30 & 84.57 \\ 
MaxLogit & ResNet-50 & 28.51 & 94.85 & 95.41 & 56.70 & 86.13 & 85.10 & 59.06 & 84.89 & 84.27 & 47.15 & 86.49 & 89.03 & 48.43 & 90.35 & 86.24 & 47.97 & 88.54 & 88.01 \\ 
KL & ResNet-50 & 26.39 & 95.11 & 95.62 & 51.87 & 86.78 & 85.45 & 55.50 & 85.34 & 84.49 & 42.84 & 87.12 & 89.30 & 46.78 & 90.51 & 86.34 & 44.67 & 88.97 & 88.24 \\ 
ViM & ResNet-50 & 20.34 & 96.24 & 96.61 & 43.10 & 89.39 & 87.86 & 52.86 & 86.61 & 85.54 & 17.18 & 93.58 & 93.48 & 32.06 & 93.55 & 90.25 & 33.11 & \textbf{91.88} & \textbf{90.75} \\ 
Mahalanobis & ResNet-50 & 39.90 & 93.76 & 94.89 & 68.36 & 84.35 & 84.93 & 73.32 & 81.46 & 81.94 & 16.05 & 94.96 & 96.11 & 41.63 & 92.73 & 90.40 & 47.85 & 89.45 & 89.65 \\ 
SSD & ResNet-50 & 36.68 & 93.81 & 94.74 & 58.82 & 86.86 & 86.82 & 67.51 & 82.67 & 82.44 & 12.38 & 96.19 & 97.02 & 40.66 & 92.30 & 89.44 & 43.21 & 90.36 & 90.09 \\ 
GradNorm & ResNet-50 & 16.75 & 96.61 & 96.83 & 32.99 & 92.21 & 91.51 & 41.85 & 89.63 & 88.59 & 31.87 & 90.78 & 92.31 & 39.83 & 90.12 & 83.85 & 32.66 & 91.87 & 90.62 \\ 
KNN & ResNet-50 & 39.81 & 93.07 & 93.77 & 56.63 & 85.71 & 85.29 & 65.30 & 82.05 & 81.33 & 17.91 & 93.69 & 94.82 & 43.20 & 91.56 & 87.98 & 44.57 & 89.22 & 88.64 \\ 
Energy & ResNet-50 & 26.38 & 95.11 & 95.62 & 51.88 & 86.78 & 85.45 & 55.50 & 85.34 & 84.49 & 42.85 & 87.12 & 89.30 & 46.78 & 90.51 & 86.34 & 44.68 & 88.97 & 88.24 \\ 
\rowcolor{Gray}
NNGuide & ResNet-50 & 14.27 & 96.89 & 96.80 & 34.56 & 90.32 & 87.95 & 42.23 & 88.48 & 86.68 & 27.38 & 90.44 & 91.04 & 35.44 & 92.35 & 87.91 & \textbf{30.78} & 91.70 & 90.08 \\ 
\midrule
MSP & MobileNet & 75.68 & 81.49 & 83.19 & 83.66 & 73.43 & 73.93 & 83.46 & 73.15 & 73.75 & 76.32 & 75.80 & 82.79 & 77.91 & 79.35 & 72.03 & 79.40 & 76.64 & 77.14 \\ 
MaxLogit & MobileNet & 78.71 & 79.15 & 80.55 & 84.84 & 71.81 & 71.85 & 84.48 & 71.64 & 71.80 & 75.50 & 75.00 & 81.86 & 77.43 & 78.23 & 69.24 & 80.19 & 75.17 & 75.06 \\ 
KL & MobileNet & 92.70 & 64.99 & 69.21 & 95.02 & 62.87 & 65.74 & 93.86 & 63.39 & 66.00 & 81.31 & 69.24 & 78.25 & 84.95 & 70.36 & 60.39 & 89.57 & 66.17 & 67.92 \\ 
ViM & MobileNet & 87.49 & 68.06 & 70.28 & 89.03 & 64.78 & 65.70 & 92.47 & 60.69 & 62.33 & 41.55 & 89.07 & 91.93 & 73.97 & 78.99 & 70.57 & 76.90 & 72.32 & 72.16 \\ 
Mahalanobis & MobileNet & 67.52 & 81.74 & 82.69 & 80.25 & 71.24 & 71.39 & 86.72 & 66.12 & 66.61 & 32.29 & 92.21 & 94.38 & 59.56 & 85.56 & 79.29 & \textbf{65.27} & 79.37 & 78.87 \\ 
SSD & MobileNet & 85.87 & 64.66 & 63.62 & 88.03 & 62.20 & 61.47 & 93.03 & 54.31 & 54.69 & 40.74 & 90.66 & 94.19 & 78.53 & 74.76 & 64.32 & 77.24 & 69.32 & 67.66 \\ 
GradNorm & MobileNet & 94.32 & 62.01 & 65.30 & 93.74 & 62.41 & 63.90 & 95.28 & 59.31 & 61.44 & 78.67 & 73.00 & 82.00 & 87.87 & 65.77 & 53.72 & 89.98 & 64.50 & 65.27 \\ 
KNN & MobileNet & 83.08 & 74.75 & 77.98 & 92.08 & 64.94 & 67.67 & 93.23 & 60.90 & 63.73 & 37.51 & 90.44 & 92.79 & 71.57 & 80.62 & 73.82 & 75.49 & 74.33 & 75.20 \\ 
Energy & MobileNet & 92.69 & 64.99 & 69.21 & 95.02 & 62.87 & 65.74 & 93.86 & 63.39 & 66.00 & 81.31 & 69.24 & 78.25 & 84.95 & 70.36 & 60.39 & 89.57 & 66.17 & 67.92 \\ 
\rowcolor{Gray}
NNGuide & MobileNet & 70.18 & 80.16 & 81.41 & 81.27 & 73.81 & 74.61 & 83.59 & 71.86 & 72.26 & 40.65 & 88.28 & 90.77 & 63.31 & 82.89 & 75.60 & 67.80 & \textbf{79.40} & \textbf{78.93} \\
\midrule
MSP & ViT & 33.17 & 92.32 & 92.40 & 63.43 & 81.97 & 82.01 & 66.79 & 80.97 & 81.06 & 58.90 & 81.96 & 87.14 & 50.76 & 87.43 & 81.09 & 54.61 & 84.93 & 84.74 \\ 
MaxLogit & ViT & 19.72 & 95.88 & 95.86 & 56.13 & 83.15 & 81.28 & 62.67 & 79.01 & 76.22 & 53.64 & 82.71 & 86.32 & 37.69 & 90.84 & 84.06 & 45.97 & 86.32 & 84.75 \\ 
KL & ViT & 19.33 & 95.99 & 95.98 & 57.33 & 82.69 & 80.81 & 64.98 & 77.90 & 75.17 & 55.74 & 82.12 & 86.03 & 37.79 & 90.83 & 84.08 & 47.03 & 85.91 & 84.41 \\ 
ViM & ViT & 4.34 & 98.96 & 98.91 & 55.11 & 85.81 & 84.50 & 64.44 & 80.93 & 79.21 & 47.21 & 86.47 & 89.91 & 24.53 & 94.77 & 91.17 & \textbf{39.13} & 89.39 & 88.74 \\ 
Mahalanobis & ViT & 6.30 & 98.58 & 98.48 & 62.70 & 86.11 & 86.09 & 69.17 & 83.13 & 83.65 & 48.42 & 88.66 & 92.61 & 28.34 & 94.97 & 92.95 & 42.99 & \textbf{90.29} & \textbf{90.76} \\ 
SSD & ViT & 10.31 & 97.66 & 97.61 & 83.16 & 72.06 & 72.34 & 87.18 & 66.35 & 67.51 & 67.40 & 80.75 & 86.47 & 43.36 & 89.25 & 83.03 & 58.28 & 81.21 & 81.39 \\ 
GradNorm & ViT & 21.85 & 95.06 & 94.45 & 56.94 & 82.79 & 79.68 & 66.08 & 78.35 & 75.14 & 53.30 & 84.18 & 87.96 & 39.91 & 89.97 & 81.55 & 47.62 & 86.07 & 83.75 \\ 
KNN & ViT & 33.94 & 93.34 & 93.60 & 75.12 & 82.53 & 81.79 & 76.53 & 80.05 & 80.60 & 54.72 & 86.11 & 90.64 & 49.59 & 90.54 & 87.04 & 57.98 & 86.52 & 86.73 \\ 
Energy & ViT & 19.33 & 95.99 & 95.98 & 57.33 & 82.69 & 80.81 & 64.98 & 77.90 & 75.17 & 55.74 & 82.12 & 86.03 & 37.79 & 90.83 & 84.08 & 47.03 & 85.91 & 84.41 \\
\rowcolor{Gray} 
NNGuide & ViT & 13.96 & 97.18 & 97.25 & 54.73 & 87.42 & 86.80 & 61.47 & 84.18 & 83.57 & 46.86 & 87.68 & 91.11 & 31.65 & 93.92 & 91.00 & 41.73 & 90.08 & 89.95 \\
\midrule
MSP & RegNet & 28.48 & 93.13 & 92.83 & 58.44 & 83.54 & 83.03 & 62.95 & 81.83 & 81.41 & 54.71 & 83.54 & 88.58 & 40.37 & 89.88 & 83.82 & 48.99 & 86.38 & 85.93 \\ 
MaxLogit & RegNet & 12.18 & 96.84 & 96.28 & 41.26 & 87.68 & 85.59 & 51.19 & 83.14 & 80.32 & 42.35 & 87.25 & 90.56 & 23.88 & 93.47 & 87.12 & 34.17 & 89.67 & 87.97 \\ 
KL & RegNet & 10.59 & 97.19 & 96.59 & 38.75 & 87.91 & 85.54 & 50.02 & 82.62 & 79.54 & 41.26 & 87.40 & 90.63 & 22.81 & 93.56 & 87.03 & 32.69 & 89.74 & 87.87 \\ 
ViM & RegNet & 3.38 & 99.22 & 99.19 & 37.13 & 90.39 & 89.18 & 52.27 & 85.02 & 83.67 & 28.39 & 93.58 & 95.77 & 18.93 & 95.78 & 92.62 & 28.02 & 92.80 & 92.09 \\ 
Mahalanobis & RegNet & 3.08 & 99.14 & 99.26 & 55.81 & 87.90 & 87.75 & 67.00 & 83.27 & 83.56 & 33.83 & 92.67 & 95.25 & 23.19 & 95.64 & 93.55 & 36.58 & 91.72 & 91.87 \\ 
SSD & RegNet & 6.45 & 98.52 & 98.70 & 65.12 & 83.77 & 83.83 & 74.48 & 77.58 & 78.13 & 43.07 & 91.30 & 94.74 & 31.84 & 92.45 & 87.86 & 44.19 & 88.73 & 88.65 \\ 
GradNorm & RegNet & 88.78 & 54.91 & 55.98 & 84.73 & 63.85 & 63.64 & 92.11 & 53.86 & 54.84 & 77.46 & 73.88 & 82.86 & 79.28 & 58.61 & 40.45 & 84.47 & 61.02 & 59.56 \\ 
KNN & RegNet & 4.82 & 98.67 & 98.67 & 49.88 & 87.55 & 84.99 & 59.94 & 84.05 & 83.05 & 30.42 & 91.36 & 93.64 & 23.63 & 95.08 & 92.34 & 33.74 & 91.34 & 90.54 \\ 
Energy & RegNet & 10.80 & 97.16 & 96.57 & 38.82 & 87.95 & 85.61 & 50.09 & 82.73 & 79.70 & 41.57 & 87.37 & 90.61 & 23.01 & 93.54 & 86.98 & 32.86 & 89.75 & 87.89 \\
\rowcolor{Gray} 
NNGuide & RegNet & 2.95 & 99.32 & 99.22 & 28.56 & 92.20 & 90.88 & 39.10 & 88.81 & 87.30 & 23.85 & 93.96 & 95.91 & 15.38 & 96.56 & 93.87 & \textbf{21.97} & \textbf{94.17} & \textbf{93.44} \\ 
\bottomrule
\end{tabular}
}
\caption{
Results on ImageNet-1k-V2 (ID) across five different OODs (\ie iNaturalist, SUN, Places, Textures, OpenImage-O).
}
\label{table:result_in1kv2_supp}
\end{table*}

\begin{table*}[t]
\centering
\resizebox{.995\linewidth}{!}{
\begin{tabular}{l|lll|lll|lll|lll|lll|lll}
\toprule
~ & \multicolumn{3}{c}{CIFAR-100} & \multicolumn{3}{c}{SVHN} & \multicolumn{3}{c}{LSUN} & \multicolumn{3}{c}{iSUN} & \multicolumn{3}{c}{ImageNet} & \multicolumn{3}{c}{Average} \\ 
~ & FPR95 & AUROC & AUPR & FPR95 & AUROC & AUPR & FPR95 & AUROC & AUPR & FPR95 & AUROC & AUPR & FPR95 & AUROC & AUPR & FPR95 & AUROC & AUPR \\ 
\midrule
MSP & 82.25 & 77.28 & 79.96 & 64.72 & 86.24 & 78.30 & 71.45 & 83.21 & 85.75 & 71.31 & 83.01 & 86.69 & 72.95 & 81.98 & 84.53 & 72.53 & 82.34 & 83.05 \\ 
MaxLogit & 82.06 & 77.95 & 80.07 & 59.83 & 88.87 & 82.24 & 64.53 & 86.84 & 88.79 & 65.87 & 86.62 & 89.51 & 67.30 & 85.42 & 87.42 & 67.92 & 85.14 & 85.61 \\ 
KL & 82.57 & 77.75 & 79.95 & 59.39 & 89.10 & 82.53 & 61.86 & 87.41 & 89.15 & 62.41 & 87.19 & 89.83 & 65.14 & 85.91 & 87.71 & 66.27 & 85.47 & 85.84 \\ 
ViM & 87.57 & 71.04 & 73.85 & 74.78 & 81.65 & 72.62 & 81.01 & 78.13 & 81.55 & 77.84 & 79.21 & 83.58 & 79.18 & 78.45 & 81.53 & 80.07 & 77.70 & 78.63 \\ 
Mahalanobis & 87.61 & 70.50 & 71.40 & 75.90 & 81.69 & 72.97 & 75.29 & 81.53 & 84.44 & 73.86 & 80.90 & 84.63 & 75.05 & 79.80 & 82.47 & 77.54 & 78.88 & 79.18 \\ 
SSD & 90.85 & 60.32 & 59.83 & 83.19 & 70.20 & 49.35 & 82.50 & 71.80 & 73.08 & 80.29 & 71.33 & 73.74 & 79.21 & 71.59 & 71.59 & 83.21 & 69.05 & 65.52 \\ 
GradNorm & 83.16 & 68.15 & 66.22 & 51.85 & 86.12 & 71.58 & 65.78 & 76.58 & 73.68 & 64.69 & 78.20 & 77.90 & 69.01 & 74.50 & 71.30 & 66.90 & 76.71 & 72.14 \\ 
KNN & 82.79 & 76.24 & 75.43 & 65.48 & 87.21 & 80.51 & 65.96 & 86.62 & 88.94 & 67.15 & 85.44 & 88.45 & 68.99 & 85.35 & 87.82 & 70.07 & 84.17 & 84.23 \\ 
Energy & 82.57 & 77.75 & 79.95 & 59.39 & 89.10 & 82.53 & 61.86 & 87.41 & 89.15 & 62.41 & 87.19 & 89.83 & 65.14 & 85.91 & 87.71 & 66.27 & 85.47 & 85.84 \\ 
\rowcolor{Gray}
NNGuide & 81.74 & 78.34 & 80.46 & 52.04 & 90.93 & 85.25 & 61.30 & 88.02 & 89.78 & 62.08 & 87.84 & 90.50 & 65.61 & 86.82 & 88.82 & \textbf{64.56} & \textbf{86.39} & \textbf{86.96} \\ 
\bottomrule
\end{tabular}
}
\caption{
Results on CIFAR-100.
}
\label{table:result_cifar100_supp}
\end{table*}

\subsection{Evaluation on the CIFAR-100 benchmark}

\paragraph{Backbone model}
A standard ResNet-18 model with the default PyTorch configuration is trained on the train fold of CIFAR-100 from scratch for 200 epochs with the SGD optimizer. We select the best model by the validation set accuracy.

\paragraph{Results}
The results on the CIFAR-100 OOD benchmark is given in Tab.~\ref{table:result_cifar100_supp}.

\begin{table*}[t]
\centering
\resizebox{.995\linewidth}{!}{
\begin{tabular}{ll|lll|lll|lll|lll|lll|lll}
\toprule
~ & OOD & \multicolumn{3}{c}{iNaturalist} & \multicolumn{3}{c}{SUN} & \multicolumn{3}{c}{Places} & \multicolumn{3}{c}{Textures} & \multicolumn{3}{c}{OpenImage-O} & \multicolumn{3}{c}{Average} \\
Components & Model & FPR95 & AUROC & AUPR & FPR95 & AUROC & AUPR & FPR95 & AUROC & AUPR & FPR95 & AUROC & AUPR & FPR95 & AUROC & AUPR & FPR95 & AUROC & AUPR \\
\midrule
KNN & ResNet-50 & 37.53 & 93.86 & 98.67 & 54.57 & 86.98 & 96.71 & 63.34 & 83.54 & 95.68 & 17.45 & 94.13 & 99.00 & 40.77 & 92.46 & 97.15 & 42.73 & 90.19 & 97.44 \\ 
KNN with average similarity & ResNet-50 & 40.80 & 93.82 & 98.71 & 54.65 & 87.84 & 96.91 & 62.34 & 85.13 & 96.14 & 19.84 & 93.91 & 98.99 & 42.16 & 92.87 & 97.42 & 43.96 & 90.71 & 97.64 \\ 
Energy & ResNet-50 & 20.98 & 96.17 & 99.19 & 47.05 & 88.91 & 97.08 & 51.15 & 87.70 & 96.85 & 39.31 & 88.90 & 97.96 & 41.56 & 92.32 & 97.06 & 40.01 & 90.80 & 97.63 \\ 
Product fusion & ResNet-50 & 18.16 & 96.64 & 99.29 & 43.62 & 89.86 & 97.27 & 50.33 & 88.07 & 96.91 & 23.58 & 92.26 & 98.42 & 31.98 & 94.02 & 97.73 & 33.53 & 92.17 & 97.92 \\ 
Sum fusion & ResNet-50 & 21.51 & 96.38 & 99.24 & 44.61 & 89.87 & 97.30 & 52.44 & 87.78 & 96.85 & 20.14 & 93.05 & 98.55 & 31.43 & 94.26 & 97.83 & 34.03 & 92.27 & 97.96 \\ 
Max fusion & ResNet-50 & 37.53 & 93.87 & 98.67 & 54.55 & 86.98 & 96.71 & 63.33 & 83.54 & 95.69 & 17.45 & 94.13 & 98.97 & 40.76 & 92.46 & 97.15 & 42.72 & 90.19 & 97.44 \\ 
Min fusion & ResNet-50 & 20.98 & 96.17 & 99.19 & 47.05 & 88.91 & 97.08 & 51.15 & 87.70 & 96.85 & 39.31 & 88.90 & 97.96 & 41.56 & 92.32 & 97.06 & 40.01 & 90.80 & 97.63 \\ 
Mahalanobis guidance & ResNet-50 & 21.35 & 96.33 & 99.25 & 53.52 & 88.72 & 97.15 & 59.53 & 86.76 & 96.69 & 19.04 & 93.84 & 98.81 & 32.49 & 94.36 & 97.93 & 37.19 & 92.00 & 97.97 \\ 
Guidance term only & ResNet-50 & 18.22 & 95.83 & 98.97 & 30.46 & 91.16 & 97.37 & 39.11 & 88.86 & 96.82 & 23.23 & 92.34 & 98.39 & 40.54 & 89.71 & 95.55 & 30.31 & 91.58 & 97.42 \\ 
W/O confidence scaling & ResNet-50 & 22.12 & 96.22 & 99.21 & 46.09 & 89.47 & 97.22 & 51.63 & 87.82 & 96.88 & 25.39 & 91.86 & 98.40 & 35.91 & 93.66 & 97.62 & 36.23 & 91.81 & 97.87 \\ 
\rowcolor{Gray}
NNGuide & ResNet-50 & 12.02 & 97.47 & 99.43 & 31.62 & 91.66 & 97.63 & 38.88 & 90.12 & 97.34 & 24.93 & 91.52 & 98.27 & 31.60 & 93.66 & 97.47 & \textbf{27.81} & \textbf{92.89} & \textbf{98.03} \\ 
\midrule
KNN & RegNet & 4.30 & 98.76 & 99.73 & 46.12 & 88.45 & 96.69 & 56.28 & 85.15 & 96.11 & 28.33 & 91.93 & 98.72 & 21.26 & 95.51 & 98.28 & 31.26 & 91.96 & 97.91 \\ 
KNN with average similarity & RegNet & 3.24 & 99.28 & 99.83 & 37.68 & 90.59 & 97.32 & 46.79 & 87.91 & 96.84 & 24.45 & 93.00 & 98.87 & 14.79 & 97.06 & 98.85 & 25.39 & 93.57 & 98.34 \\ 
Energy & RegNet & 6.68 & 98.28 & 99.57 & 29.41 & 91.88 & 97.85 & 40.51 & 87.97 & 96.72 & 30.85 & 91.48 & 98.68 & 16.19 & 95.81 & 98.09 & 24.73 & 93.08 & 98.18 \\ 
Product fusion & RegNet & 2.51 & 99.44 & 99.87 & 25.98 & 93.43 & 98.21 & 36.91 & 90.50 & 97.54 & 20.25 & 95.07 & 99.26 & 10.49 & 97.97 & 99.21 & 19.23 & 95.28 & 98.82 \\ 
Sum fusion & RegNet & 3.55 & 99.20 & 99.81 & 24.74 & 93.69 & 98.39 & 35.67 & 90.54 & 97.53 & 23.16 & 94.44 & 99.20 & 10.99 & 97.55 & 98.97 & 19.62 & 95.08 & 98.78 \\ 
Max fusion & RegNet & 4.30 & 98.76 & 99.73 & 46.09 & 88.46 & 96.69 & 56.26 & 85.17 & 96.11 & 28.32 & 91.94 & 98.73 & 21.25 & 95.52 & 98.28 & 31.24 & 91.97 & 97.91 \\ 
Min fusion & RegNet & 6.68 & 98.28 & 99.57 & 29.41 & 91.88 & 97.85 & 40.51 & 87.97 & 96.72 & 30.85 & 91.48 & 98.68 & 16.19 & 95.81 & 98.09 & 24.73 & 93.08 & 98.18 \\ 
Mahalanobis guidance & RegNet & 1.24 & 99.61 & 99.92 & 34.07 & 92.65 & 98.23 & 47.87 & 88.98 & 97.31 & 19.77 & 95.37 & 99.36 & 12.45 & 97.70 & 99.16 & 23.08 & 94.86 & 98.80 \\ 
Guidance term only & RegNet & 2.27 & 99.43 & 99.86 & 30.68 & 91.71 & 97.66 & 40.14 & 89.34 & 97.20 & 19.59 & 94.21 & 99.11 & 18.82 & 95.78 & 98.27 & 22.30 & 94.09 & 98.42 \\ 
W/O confidence scaling & RegNet & 2.35 & 99.50 & 99.89 & 23.99 & 94.09 & 98.46 & 34.06 & 91.50 & 97.83 & 19.84 & 95.23 & 99.31 & 8.68 & 98.30 & 99.35 & 17.78 & 95.72 & 98.97 \\ 
\rowcolor{Gray}
NNGuide & RegNet & 1.83 & 99.57 & 99.90 & 21.58 & 94.43 & 98.58 & 31.47 & 91.87 & 97.92 & 17.00 & 95.82 & 99.42 & 10.79 & 97.73 & 99.09 & \textbf{16.53} & \textbf{95.89} & \textbf{98.98} \\
\bottomrule
\end{tabular}
}
\caption{
Ablation study on the components of NNGuide. The ID is ImageNet-1k.
}
\label{table:ablation_components_suppv1}
\end{table*}

\begin{table*}[t]
\centering
\resizebox{.995\linewidth}{!}{
\begin{tabular}{ll|lll|lll|lll|lll|lll|lll}
\toprule
~ & OOD & \multicolumn{3}{c}{iNaturalist} & \multicolumn{3}{c}{SUN} & \multicolumn{3}{c}{Places} & \multicolumn{3}{c}{Textures} & \multicolumn{3}{c}{OpenImage-O} & \multicolumn{3}{c}{Average} \\
Components & Model & FPR95 & AUROC & AUPR & FPR95 & AUROC & AUPR & FPR95 & AUROC & AUPR & FPR95 & AUROC & AUPR & FPR95 & AUROC & AUPR & FPR95 & AUROC & AUPR \\
\midrule
KNN & ResNet-50 & 39.81 & 93.07 & 93.77 & 56.63 & 85.71 & 85.29 & 65.30 & 82.05 & 81.33 & 17.91 & 93.69 & 94.82 & 43.20 & 91.56 & 87.98 & 44.57 & 89.22 & 88.64 \\ 
KNN with average similarity & ResNet-50 & 44.61 & 92.66 & 93.80 & 58.47 & 86.10 & 85.70 & 65.82 & 83.14 & 82.60 & 21.11 & 93.24 & 94.67 & 46.01 & 91.61 & 88.80 & 47.20 & 89.35 & 89.11 \\ 
Energy & ResNet-50 & 26.38 & 95.11 & 95.62 & 51.88 & 86.78 & 85.45 & 55.50 & 85.34 & 84.49 & 42.85 & 87.12 & 89.30 & 46.78 & 90.51 & 86.34 & 44.68 & 88.97 & 88.24 \\ 
Product fusion & ResNet-50 & 23.45 & 95.79 & 96.20 & 48.01 & 88.03 & 86.44 & 54.37 & 85.93 & 84.91 & 25.31 & 91.25 & 91.69 & 37.77 & 92.68 & 89.23 & 37.78 & 90.74 & 89.69 \\ 
Sum fusion & ResNet-50 & 26.44 & 95.55 & 96.05 & 48.58 & 88.14 & 86.72 & 56.06 & 85.72 & 84.84 & 21.55 & 92.24 & 92.38 & 36.49 & 93.06 & 89.81 & 37.82 & 90.94 & 89.96 \\ 
Max fusion & ResNet-50 & 39.81 & 93.07 & 93.77 & 56.63 & 85.71 & 85.30 & 65.30 & 82.05 & 81.33 & 17.91 & 93.69 & 94.69 & 43.20 & 91.56 & 87.98 & 44.57 & 89.22 & 88.61 \\ 
Min fusion & ResNet-50 & 26.38 & 95.11 & 95.62 & 51.88 & 86.78 & 85.45 & 55.50 & 85.34 & 84.49 & 42.85 & 87.12 & 89.30 & 46.78 & 90.51 & 86.34 & 44.68 & 88.97 & 88.24 \\ 
Mahalanobis guidance & ResNet-50 & 26.57 & 95.46 & 96.15 & 57.85 & 86.62 & 86.15 & 63.44 & 84.35 & 84.14 & 20.16 & 93.03 & 93.60 & 37.02 & 93.10 & 90.29 & 41.01 & 90.51 & 90.07 \\ 
Guidance term only & ResNet-50 & 20.15 & 95.41 & 94.93 & 32.50 & 90.55 & 87.73 & 41.43 & 88.10 & 85.54 & 24.97 & 91.89 & 92.10 & 42.83 & 88.93 & 81.63 & 32.37 & 90.98 & 88.38 \\ 
W/O confidence scaling & ResNet-50 & 28.12 & 95.23 & 95.85 & 50.65 & 87.52 & 86.23 & 56.27 & 85.58 & 84.78 & 27.54 & 90.73 & 91.52 & 41.89 & 92.18 & 88.82 & 40.89 & 90.25 & 89.44 \\ 
\rowcolor{Gray}
NNGuide & ResNet-50 & 14.27 & 96.89 & 96.80 & 34.56 & 90.32 & 87.95 & 42.23 & 88.48 & 86.68 & 27.38 & 90.44 & 91.04 & 35.44 & 92.35 & 87.91 & \textbf{30.78} & \textbf{91.70} & \textbf{90.08} \\ 
\midrule
KNN & RegNet & 4.82 & 98.67 & 98.67 & 49.88 & 87.55 & 84.99 & 59.94 & 84.05 & 83.05 & 30.42 & 91.36 & 93.64 & 23.63 & 95.08 & 92.34 & 33.74 & 91.34 & 90.54 \\ 
KNN with average similarity & RegNet & 3.99 & 99.14 & 99.05 & 42.51 & 89.34 & 87.03 & 51.81 & 86.42 & 85.28 & 27.44 & 92.17 & 94.04 & 17.45 & 96.51 & 94.25 & 28.64 & 92.72 & 91.93 \\ 
Energy & RegNet & 10.80 & 97.16 & 96.57 & 38.82 & 87.95 & 85.61 & 50.09 & 82.73 & 79.70 & 41.57 & 87.37 & 90.61 & 23.01 & 93.54 & 86.98 & 32.86 & 89.75 & 87.89 \\ 
Product fusion & RegNet & 3.79 & 99.15 & 99.05 & 33.50 & 91.18 & 89.16 & 45.05 & 87.41 & 85.87 & 27.73 & 93.23 & 95.15 & 15.07 & 96.95 & 94.86 & 25.03 & 93.58 & 92.82 \\ 
Sum fusion & RegNet & 5.62 & 98.63 & 98.44 & 33.27 & 90.62 & 89.04 & 45.00 & 86.35 & 84.20 & 32.06 & 91.58 & 94.12 & 16.77 & 96.04 & 92.60 & 26.55 & 92.64 & 91.68 \\ 
Max fusion & RegNet & 4.82 & 98.67 & 98.67 & 49.88 & 87.55 & 85.00 & 59.94 & 84.06 & 83.06 & 30.42 & 91.37 & 93.64 & 23.63 & 95.08 & 92.35 & 33.74 & 91.35 & 90.54 \\ 
Min fusion & RegNet & 10.80 & 97.16 & 96.57 & 38.82 & 87.95 & 85.61 & 50.09 & 82.73 & 79.70 & 41.57 & 87.37 & 90.61 & 23.01 & 93.54 & 86.98 & 32.86 & 89.75 & 87.89 \\ 
Mahalanobis guidance & RegNet & 1.96 & 99.39 & 99.44 & 43.34 & 90.25 & 89.77 & 56.84 & 85.71 & 85.30 & 26.83 & 93.78 & 95.92 & 17.22 & 96.73 & 94.99 & 29.24 & 93.17 & 93.09 \\ 
Guidance term only & RegNet & 2.74 & 99.32 & 99.21 & 34.15 & 90.65 & 88.41 & 43.91 & 88.01 & 86.62 & 22.29 & 93.48 & 95.26 & 21.10 & 95.13 & 91.53 & 24.84 & 93.32 & 92.21 \\ 
W/O confidence scaling & RegNet & 3.79 & 99.19 & 99.11 & 32.71 & 91.72 & 90.27 & 43.47 & 88.32 & 86.95 & 27.65 & 93.23 & 95.27 & 13.81 & 97.29 & 95.50 & 24.29 & 93.95 & 93.42 \\ 
\rowcolor{Gray}
NNGuide & RegNet & 2.95 & 99.32 & 99.22 & 28.56 & 92.20 & 90.88 & 39.10 & 88.81 & 87.30 & 23.85 & 93.96 & 95.91 & 15.38 & 96.56 & 93.87 & \textbf{21.97} & \textbf{94.17} & \textbf{93.44} \\
\bottomrule
\end{tabular}
}
\caption{
Ablation study on the components of NNGuide. The ID is ImageNet-1k-V2, where test ID samples undergo natural distribution shift.
}
\label{table:ablation_components_suppv2}
\end{table*}

\subsection{Ablation}

\subsubsection{Compatibility to other classifier-based confidence scores}

\begin{figure*}[t]
\centering
\includegraphics[width=.995\linewidth]{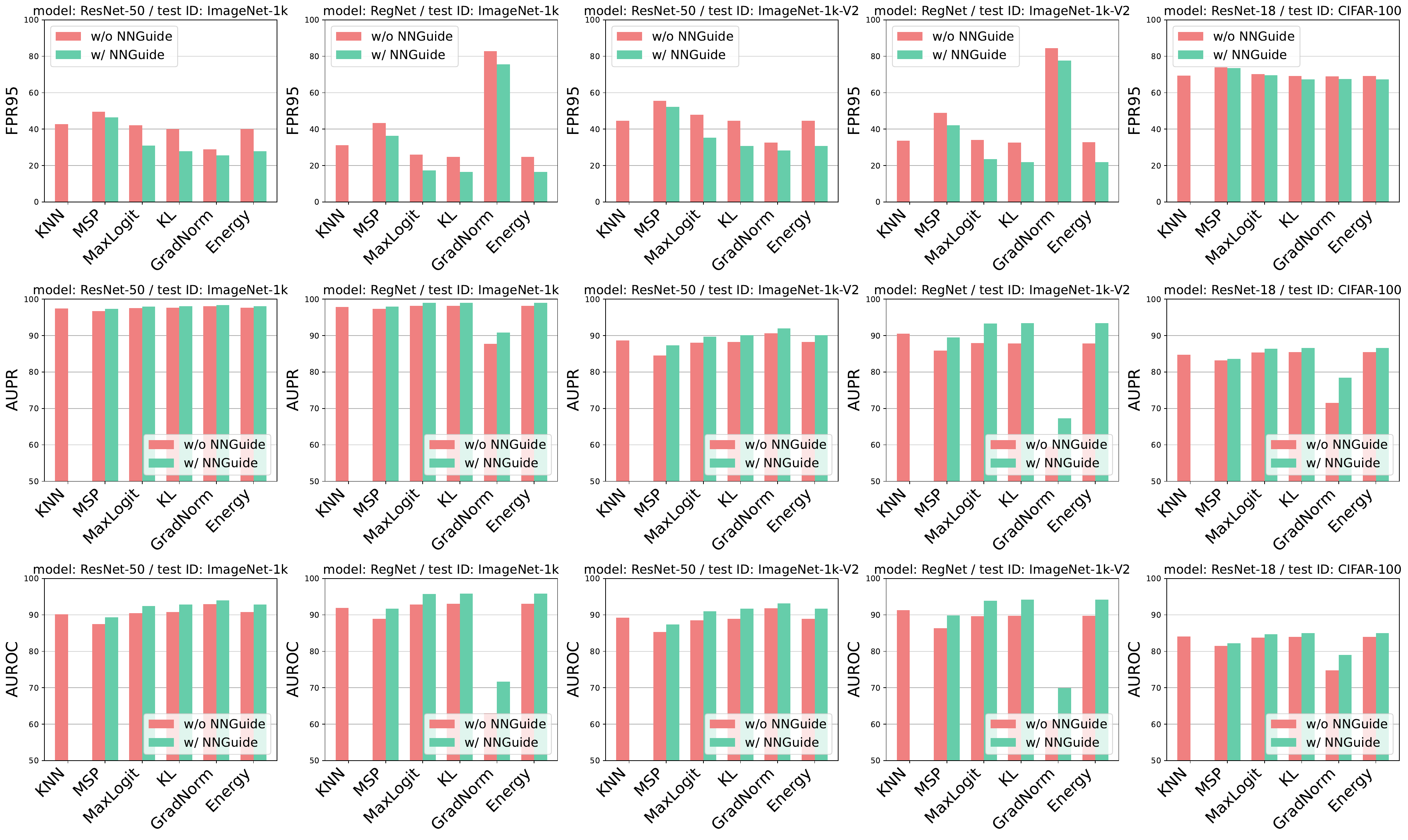}
\caption{
The compatibility to other classifier-based scores. The average performance across five different OODs is reported.
}
\label{fig:compatibility_other_confs_supp}
\end{figure*}

\paragraph{Note}
Note that the all confidence scores we use have their range in $[0, \infty)$. Particularly, the MaxLogit and Energy confidence scores satisfy this range as long as the maximum unit of the logit is greater than or equal to $0$.

\paragraph{Results}
The result is given in Fig.~\ref{fig:compatibility_other_confs_supp}.

\subsubsection{Ablation on the components of NNGuide}

\paragraph{Descriptions}
We describe in detail the detection methods used in the ablation study of NNGuide components (\ie Tab.~\ref{table:ablation_components}). Let $f$ be a neural network that outputs the classification logit, and $\phi$ the feature extractor inside the network. Let $S_{base}$ be the score function that computes the base confidence (\ie negative Energy) score given an input. Let $\{\mathbf{x}_1, \dots, \mathbf{x}_n\}$ be the ID bank set, and $\{\mathbf{z}_1, \dots, \mathbf{z}_n\}$ the corresponding features (\ie $\mathbf{z}_i = \phi(\mathbf{x}_i)$) with $s_i = S_{base}(\mathbf{x}_i)$.
Let $\mathbf{x}$ be a test input and $\mathbf{z}$ its extracted feature.  

\begin{itemize}
\item 
\textbf{KNN}: The KNN score is computed by
\begin{equation}
S_{KNN}(\mathbf{x}) = \simop ( \mathbf{z}_{(k)}, \mathbf{z})
\end{equation}
where the ordered indices $(i)$ satisfy
\begin{equation}
\simop ( \mathbf{z}_{(1)}, \mathbf{z}) \geq \cdots \geq  \simop ( \mathbf{z}_{(n)}, \mathbf{z}).
\end{equation}
\item
\textbf{KNN with average similarity}: The score of KNN with the average similarity slightly modifies the original KNN by
\begin{equation}
\label{eq:knn_avg}
S_{KNN-avg}(\mathbf{x}) = \frac{1}{k} \sum_{i=1}^k \simop ( \mathbf{z}_{(i)}, \mathbf{z}).
\end{equation}
\item
\textbf{Energy}: The (negative) Energy score is computed by
\begin{equation}
S_{base}(\mathbf{x}) = \log \sum_{c=1}^K \exp f_c(\mathbf{x})
\end{equation}
\item
\textbf{Naive fusion}: The basic fusion of KNN and the base confidence is performed by either of $S_{KNN}(x) \cdot S_{base}(x)$, $c_{sum} S_{KNN}(x) + S_{base}(x)$, $c_{\max} S_{KNN}(x) + S_{base}(x)$, and $c_{\min} S_{KNN}(x) + S_{base}(x)$. Note that the coefficients $c_{sum}$, $c_{\max}$, and $c_{\min}$ are the coefficients to min-max normalize the scores to manually balance the importance of the two scores. Note that min-max normalization is done based on the bank set.
\item
\textbf{Mahalanobis guidance}: 
In this case, the guidance term $G(x)$ is given as the Mahalanobis score
\begin{equation}
G(\mathbf{x}) = \exp(- \min_{c=1}^K (\mathbf{x} - \boldsymbol{\mu}_c)^T\Sigma^{-1}(\mathbf{x} - \boldsymbol{\mu}_c)/(2\cdot d))
\end{equation}
where $\boldsymbol{\mu}_c$ is the mean of $c$-th class features, $\Sigma$ is the shared covariance matirx, and $d$ is the dimension of feature. (Without the 
\item
\textbf{Guidance-term only}: In this case, the score function in utilization is 
\begin{equation}
S_{guide-only}(\mathbf{x}) = G(\mathbf{x})
\end{equation}
where $G(\mathbf{x})$ is the nearest-neighbor guidance term given in \eqref{eq:sim_ensemble}.
\item
\textbf{Without confidence scaling}: In this case, the detection score function is computed by without the scaling term in the nearest neighbor similarities. Namely,
\begin{equation}
S_{w/o-scale}(\mathbf{x}) = S_{base}(\mathbf{x}) \cdot G_{w/o-scale}(\mathbf{x})
\end{equation}
where $G_{w/o-scale}(\mathbf{x}) = S_{KNN-avg}(\mathbf{x})$ as given in \eqref{eq:knn_avg}.
\end{itemize}

\paragraph{Results}
We perform ablation study on the NNGuide components on both ImageNet-1k and ImageNet-1k-V2 across ResNet-18 and RegNet. The results are given in Tab.~\ref{table:ablation_components_suppv1} and \ref{table:ablation_components_suppv2}.

\begin{figure*}[t]
\centering
\includegraphics[width=.9995\linewidth]{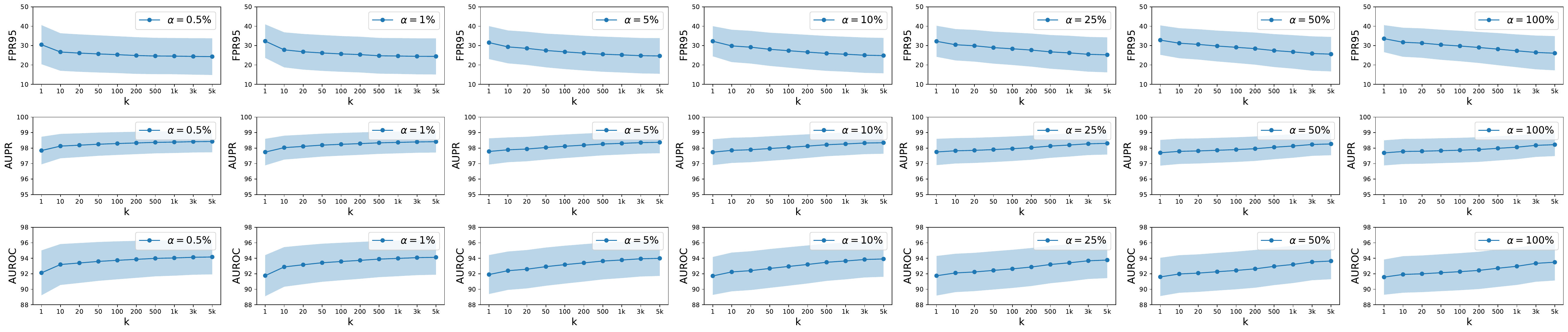}
\caption{
The performance of NNGuide across the number $k$ of nearest neighbors and the sampling ratio $\alpha$.
}
\label{fig:hyperparameters_supp}
\end{figure*}

\subsubsection{Analysis of the hyperparameters}
Fig.~\ref{fig:hyperparameters_supp} indicates the full analysis of the nearest neighbor hyperparameters $\alpha$ and $k$.

\begin{figure*}[t]
\centering
\includegraphics[width=.995\linewidth]{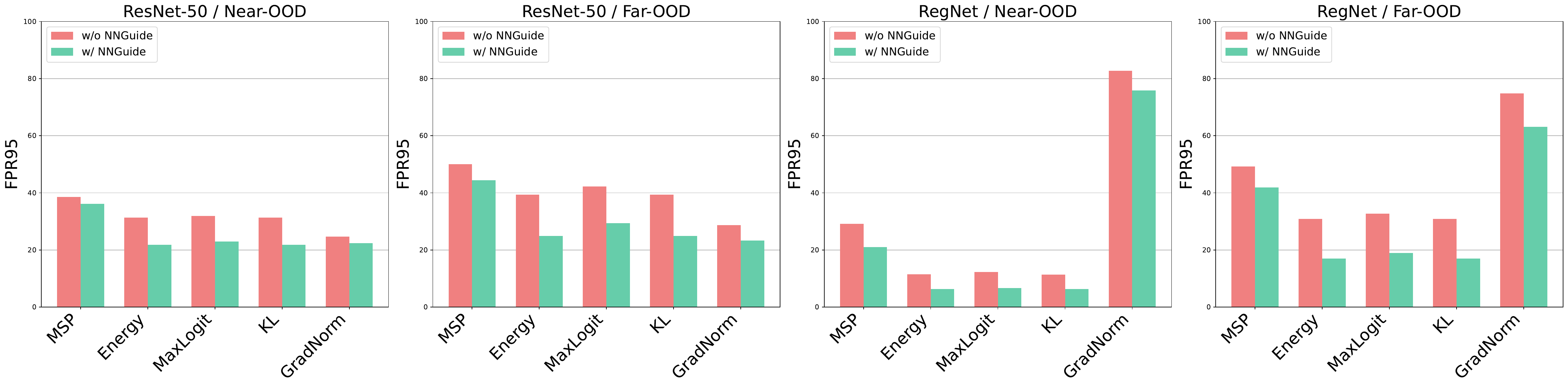}
\caption{
The improvement by our proposed nearest neighbor guidance against far-OOD data (Textures) and near-OOD data (iNaturalist and OpenImage-O). The ID data is ImageNet-1k and the model is ResNet-50.
}
\label{fig:far_vs_near}
\end{figure*}

\subsubsection{Dataset analysis: near-OOD vs far-OOD}
Fig.~\ref{fig:far_vs_near} indicates that the nearest neighbor guidance improves the base confidence score against both far-OOD data (\ie Textures) and near-OOD data (\ie iNaturalist and OpenImage-O). Notably, the improvement is more significant against the far-OOD data. This is expected by Theorem \ref{thm:theory}. Overall, NNGuide achieves balanced robustness against both far-OOD and near-OOD. 

\section{Additional Experiments}

\subsection{Experiments with CIDER}

\begin{table}[t]
\centering
\resizebox{.95\linewidth}{!}{
\begin{tabular}{lllllllllllll}
\toprule
~ & \multicolumn{2}{c}{SVHN} & \multicolumn{2}{c}{Places365} & \multicolumn{2}{c}{iSUN}  & \multicolumn{2}{c}{Texture}  & \multicolumn{2}{c}{LSUN}  & \multicolumn{2}{c}{AVG}  \\ 
~ & FPR95 & AUROC & FPR95 & AUROC & FPR95 & AUROC & FPR95 & AUROC & FPR95 & AUROC & FPR95 & AUROC \\ 
\hline
CIDER & 21.42 & 94.87 & 82.16 & 67.32 & 63.75 & 82.25 & 32.96 & 92.49 & 9.38 & 98.13 & 41.93 & 87.01 \\ 
CIDER-NNGuide & 21.52 & 94.97 & 81.95 & 67.40 & 58.71 & 84.65 & 30.18 & 93.24 & 9.75 & 98.00 & \textbf{40.42} & \textbf{87.65} \\ 
\bottomrule
\end{tabular}
}
\caption{
The results of NNGuide with CIDER on CIFAR-100 (ID).
}
\label{table:cider}
\end{table}

To analyze the compatibility of NNGuide with the state-of-the-art trainig method CIDER \cite{ming2022exploit} that is particularly effective for the KNN score, we implemented NNGuide based on the official Github repository of CIDER\footnote{\url{https://github.com/deeplearning-wisc/cider}} using the provided model weights therefrom. Tab.~\ref{table:cider} shows the performance of NNGuide with CIDER on CIFAR-100, indicating that NNGuide is compatible to CIDER. We note that, to compute base energy score for NNGuide, we defined the classifier weights by the class-wise means, and used $k=100$ with $\alpha = 1\%$.

\subsection{Experiments on ImageNet-O with ViM}

\begin{table}[t]
\centering
\resizebox{.65\linewidth}{!}{
\begin{tabular}{l ccccc}
\hline
OOD data & KNN & Energy & ViM & NNGuide-Energy & NNGuide-ViM \\
\hline 
ImageNet-O & 51.90 / 89.16 & 41.30 / 90.46 & 36.75 / 92.55 & 41.65 / 91.04 & \textbf{33.10} / \textbf{92.96} \\
\hline
\end{tabular}
}
\caption{
The result of NNGuide with ViM on ImageNet-O in (FPR95$\downarrow$ / AUROC$\uparrow$). Here, ID is the ImageNet-1k, and the backbone is ViT-B-P16-384.
}
\label{table:imageneto}
\end{table}

To evaluate on ImageNet-O \cite{hendrycks2021natural}, we applied NNGuide on ViM utilizing the implementation from the official Github repository of ViM\footnote{\url{https://github.com/vim/vim}}. 
We used the ViT-B-P16-384 backbone from \texttt{mmcls} as it achieves the SOTA with ViM.
Using the same provided ImageNet-1k train feature bankset with $k{=}100$ and the sampling ratio $\alpha=10\%$, we obtained the result in Tab.~\ref{table:imageneto}, where `NNGuide-Energy' indicates the application of NNGuide with the base score being the negative energy (\ie $S_{base}(\mathbf{x}) = - \text{Energy}(\mathbf{x})$), while `NNGuide-ViM' denotes NNGuide with the base score being ViM (\ie $S_{base}(\mathbf{x}) = - \text{ViM}(\mathbf{x})$). As the ImageNet-O is specifically designed to weaken the classifier-based confidence, both the vanilla Energy and our NNGuide-Energy perform poorly for OOD detection. 
For the ImageNet-O dataset, the integration of ViM with the ViT architecture proves to be exceptionally effective. The adversarial nature of ImageNet-O primarily targets classifier confidence and convolutional networks. This makes the combination of ViT and ViM especially robust against ImageNet-O as ViT is non-convolutional and ViM does not rely solely on raw confidence. Accordingly, the result in Tab.~\ref{table:imageneto} indicates that NNGuide's effectiveness is significantly enhanced when used in conjunction with ViM.

\subsection{Experiments on CIFAR-10}

\begin{table}[t]
\centering
\resizebox{.95\linewidth}{!}{
\begin{tabular}{l | ccc | ccccc}
\toprule
~ & \multicolumn{3}{c}{near-OOD}  & \multicolumn{5}{c}{far-OOD} \\ 
Method & CIFAR-100 & TIN & Average (near-OOD) & MNIST & SVHN & Texture & Places365 & Average (far-OOD) \\ 
\hline
Energy & 51.46 / 86.15 & 45.02 / 88.58 & 48.24 / 87.36  & 44.50 / 90.59 & 44.94 / 88.39 & 48.32 / 86.85 & 41.88 / 89.60 & 44.91 / 88.86 \\ 
KNN & 52.49 / 89.55 & 46.66 / 91.41  & 49.58 / 90.48 & 50.08 / 91.63 & 33.32 / 95.13 & 46.01 / 92.77  & 43.78 / 91.82  & 43.30 / 92.83 \\ 
NNGuide & 51.54 / 86.64 & 43.99 / 89.07 & 47.77 / 87.86 & 47.43 / 89.82 & 43.64 / 89.62 & 46.91 / 88.44 & 40.62 / 90.39 & 44.65 / 89.57 \\
\bottomrule 
\end{tabular}
}
\caption{The result of NNGuide on CIFAR-10 (ID) in (FPR95$\downarrow$ / AUROC$\uparrow$)}
\label{table:cifar10}
\end{table}

NNGuide is particularly effective for large-scale data. For a small-scale data with few classes such as CIFAR-10, ID features are already well separated class-wise. Hence KNN can already well detect near-OOD in a fine-grained manner without NNGuide, indicated by its performance similar to Energy. 
The result shown in Tab.~\ref{table:cifar10} (attained by using the official OpenOOD Github repository\footnote{\url{https://github.com/Jingkang50/OpenOOD}}) verifies our claim.
Note that the above result is obtained with $\alpha = 10\%$ and $k=10$.

\clearpage
\twocolumn

\end{document}